\documentclass[twoside]{article}

\usepackage[accepted]{aistats2019}
\usepackage[utf8]{inputenc} 
\usepackage[T1]{fontenc}    
\usepackage{hyperref}       
\usepackage{url}            
\usepackage{booktabs}       
\usepackage{amsfonts}       
\usepackage{nicefrac}       
\usepackage{microtype}
\usepackage{microtype}
\usepackage{graphicx}
\usepackage{subfigure}
\usepackage{booktabs} 
\usepackage{hyperref}
\usepackage[utf8]{inputenc} 
\usepackage[T1]{fontenc}    
\usepackage{hyperref}       
\usepackage{url}            
\usepackage{booktabs}       
\usepackage{amsfonts}       
\usepackage{nicefrac}       
\usepackage{microtype}      
\usepackage{graphicx,dsfont}
\usepackage{amsmath, amssymb, amsthm}
\usepackage{custom_aliases}
\usepackage{natbib}
\usepackage{enumerate}
\usepackage{algorithm}
\usepackage{algpseudocode}
\usepackage{times}
\setcitestyle{numbers,square,comma}

\usepackage{pgfplots}

\usepackage{tikz}
\usetikzlibrary{matrix}
\usetikzlibrary{shapes,arrows}
\usetikzlibrary{datavisualization, fit}
\usetikzlibrary{datavisualization.formats.functions}
\tikzset{
  block/.style        = {draw, thick, rectangle, minimum height = 3em, minimum width = 3em},
  observedvar/.style  = {draw, circle, node distance = 2cm},
  causalvar/.style    = {draw, circle, node distance = 2cm},
  intervention/.style = {decorate, decoration={snake, segment length=2mm, amplitude=0.5mm}},
}

\usepackage{float}
\newfloat{algorithm}{t}{lop}

\DeclareMathOperator{\EE}{\mathbb{E}}
\DeclareMathOperator{\PP}{\mathbb{P}}
\DeclareMathOperator{\R}{\mathbb{R}}
\DeclareMathOperator{\N}{\mathbb{N}}

\newtheorem{theorem}{Theorem}
\newtheorem{corollary}[theorem]{Corollary}
\newtheorem{lemma}[theorem]{Lemma}
\newtheorem{proposition}[theorem]{Proposition}

\usepackage{url}

\begin{document}
%

%

\twocolumn[

\aistatstitle{Bridging the gap between regret minimization and best arm identification, with application to A/B tests}

\aistatsauthor{ R\'emy Degenne \And Thomas Nedelec \And  Cl\'ement Calauz\`enes \And Vianney Perchet }

\aistatsaddress{ LPSM, Universit\'e Paris Diderot, \\
CMLA, ENS Paris Saclay
 \And Criteo AI Lab,\\  CMLA, ENS Paris Saclay\\
  \And Criteo AI Lab \And  CMLA, ENS Paris Saclay,\\
 Criteo AI Lab} ]

%
%
%



\begin{abstract}

State of the art online learning procedures focus  either on selecting the best alternative (``best arm identification'') or on minimizing the cost (the ``regret''). We merge these two objectives by providing the theoretical analysis of cost minimizing algorithms that are also $\delta$-PAC (with a proven guaranteed bound on the decision time), hence fulfilling at the same time regret minimization and best arm identification. This analysis sheds light on the common observation that ill-callibrated UCB-algorithms minimize regret while still identifying quickly the best arm.

We also extend these results to the non-iid case faced by many practitioners. This provides a technique to make cost versus decision time compromise when doing adaptive tests with applications ranging from website A/B testing to clinical trials.

\end{abstract}


\section*{Introduction}
With the growing use of personalization and machine learning techniques on user-facing systems, randomized experiments -- or A/B tests -- have become a standard tool to evaluate the performances of different versions of the systems. Two of the main drawbacks of such experiment are its possible length and its cost, as it commonly takes few weeks or even months to ensure a statistically founded decision. Thus, lots of attention have been paid to the complexity of the underlying statistical tests \cite{pekelis2015new,kaufmann2014, Zhao2016,Johari2017,Yang2017,Zhao2016}.

These approaches have indubitable practical interest, but  are often limited to quite simple A/B frameworks because of their restrictive assumptions. The first restrictive one is the assumption that a good procedure should minimize the time needed to take a statistically significant decision \citep{kaufmann2014, Zhao2016}. In many situations, the main objective of practitioners is to minimize the overall cost of A/B testing without preventing him to take the right decision given a certain time budget.

Another traditional assumption in the online learning literature is that outcomes arriving over time are independent \citep{kaufmann2014, Zhao2016}. Numerous common scenarii do not satisfy this and practitioners cannot benefit from the statistically efficient methods available in the iid setting. For instance, when an online retailer wants to A/B test two versions of its mobile application, it may not be able to consider the purchases as independents when customers are buying recurrently (or clicking repeatedly in the CTR optimization problem).

To address the first limitation, we exhibit a well-tuned variant of the UCB algorithm \cite{auer2002finite} that is both able to take a $\delta$-PAC decision in finite time and reaches a low regret. This algorithm can be taken as a tool for practitioners to interpolate between the tasks of \textit{best arm identification} and \textit{regret minimization}. Even if this objective has been briefly mentioned and/or observed empirically \cite{jamieson2014lil}, we provide the first theoretical analysis of such algorithm.

To handle the second limitation, we extend the ideas developed in the iid case to a more complex setting that can handle units arriving through time and delivering rewards continuously during the  test. We provide sample efficient statistical decision rules and  guarantees to take decisions in finite time in this setting, highlighting the clear trade-off between \textit{regret minimization} and \textit{best arm identification}, even on more complex settings. 


\subsection*{Regret vs Best-Arm Identification in iid setting}
\textsl{Framework.}
We consider the classical multi-armed bandit problem with $K\geq 2$ arms or ``population''. At each time step $t$, the agent chooses an arm $i \in [K]:=\{1, \ldots, K\}$ and observes a reward drawn from an unknown distribution $r^i_t$ with expectation $r^i$. We assume that each $r^i_t$ are $\sigma^2$-subGaussian, where the variance (proxy) $\sigma^2$ is known. We denote  by $\pi_n \in [K]$ the sequence of random variable  indicating which arm to pull at time $n\in\N$.

\textsl{Objective.} We consider both natural objectives of bandit problems. The first one corresponds to \textit{regret minimization}. It consists in  minimizing the cumulative regret
\begin{equation*}
R(T) = T\max\big\{r^i \, ; i \in [K] \big\} - \mathbb{E}\sum_{t = 1}^{T}r^{\pi_t}_t
\end{equation*}
In \textit{A/B testing}, when $K=2$, minimizing the regret is the same as minimizing the cost of testing a new technology or the impact of a clinical trial on patients. 

The second objective, matching the problem of \textit{best arm identification} with \textit{fixed confidence}, is to design an algorithm for a given confidence level $\delta$, that minimizes the worst-case number of sample $T_\delta$  needed for the algorithm to finish and to return the optimal arm with probability $1-\delta$. Using an algorithm for best arm identification in an A/B test gives a guarantee on the amount time necessary before being able to to take a statistically significant decision.

Intuitively, an algorithm that is optimal for regret minimization is sub-optimal for best arm identification because its exploration is too slow. The opposite is also true since the exploration of optimal best arm identification algorithms is too aggressive for regret minimization. 

We aim at studying a family of algorithms that  interpolate between these two objectives. Informally, with $\delta \in [0,1]$,  our objective is to design algorithms for which with probability $1-\delta$, for all bandit problems $P$ in some class, the worse arm is discarded after $T_\delta$ stages and we have both 
$$
T_\delta \leq f(P,\delta)
\quad \text{ and } \quad R(T_\delta) \leq g(P,\delta)\, .
$$
The values $f(P,\delta)$ and $g(P,\delta)$ characterize the performances of the algorithm and should be as small as possible.


\subsection*{Literature review}

\textsl{Regret minimization.} This objective has been extensively studied in the bandit literature since the seminal paper of \cite{thompson1933likelihood}. We mention two particular classes of well-known algorithms that we will use throughout the paper. 
\begin{itemize}
\item[--] The \textit{Upper Confidence Bound} (UCB) algorithm introduced in \cite{katehakis1995sequential,auer2002finite} decides which arms to consider depending on the respective empirical means and an error term depending on the number of pulls of each arm. Its regret, in the case of two arms with Gaussian rewards, is equal to $2\log(T)/\Delta$ where $\Delta = |r^\cA - r^\cB|$ is the gap between the mean of the two arms  \cite{auer2002finite}. 
\item[--]The other class of algorithms we consider is known as \textit{Explore Then Commit} (ETC)\cite{perchet2016batched,garivier2016explore}. They are first considering  stages of pure exploration before exploiting the arm with the highest empirical mean. The algorithm consists in tuning the switching times between the stages. Its regret in the case of two arms with Gaussian rewards is of order $4\log(T)/\Delta$.
\end{itemize}  We recall that ETC is necessarily sub-optimal for \textit{regret minimization}, as in the case of Gaussian rewards there exists a sub-optimal additional and multiplicative factor 2 \cite{garivier2016explore}.

\textsl{Best arm identification.} The problem of best arm identification \cite{mannor2004sample}, can be cast in two main settings depending on the constraint  imposed on the system: \begin{itemize}
\item[--] \textit{fixed budget} \cite{audibert2010best,bubecka2011pure}  where a total number of samples $T \in \mathbb{N}$ is given and the goal is to minimize the error probability at time T; 
\item[--]\textit{fixed confidence} \cite{mannor2004sample,carpentier2016tight} where the goal is to minimize the total number of stages used to return the best arm with probability $1-\delta$.
\end{itemize} In the fixed confidence setting there are two main ways to evaluate the sample complexity of the algorithm : \textit{the average sample complexity} studied in \cite{mannor2004sample,carpentier2016tight,kaufmann2014,kaufmann2016complexity} where the goal is to minimize the expected time of decision and \textit{the worst case sample complexity} studied in \cite{even2006action,karnin2013almost,jamieson2014lil} where the objective is to have a quantity $T_{\delta} \in \mathbb{N}$ as low as possible such that with probability $1-\delta$, the algorithm makes no mistake and the time of decision $\tau_d$ is below $T_{\delta}$. In the case of two arms, the optimal sampling strategy is to sample each arm uniformly and stop with a criterion similar to the one used in ETC \cite{kaufmann2014}.

\textsl{A/B testing} Most of the statistical literature on A/B testing \cite{kaufmann2014, Zhao2016} has focused on the objective of minimizing the time necessary to take a statistical sufficient decision and to the best of our knowledge, there exists no work theoretically  interpolating between the objectives of best arm identification and regret minimization.


\section{Simultaneous Best-arm Identification and Regret Minimization }\label{SectionBothWorlds}

In this section, we construct a family of algorithms that minimizes regret while being $\delta$-PAC (with a proven guaranteed bound on the decision time), hence fulfilling at the same time the regret minimization and best-arm identification. For the sake of clarity, we are going to assume that there are only $k=2$ populations, as in A/B testing, denoted by $\cA$ and $\cB$. The results for the general case $K>2$  can actually be deduced almost immediately from those when $K=2$.

Let us first recall that the algorithm with lowest decision time\cite{kaufmann2016complexity}, given the variance of arms are identical, is ETC, which pulls both arms uniformly and, after pulling each arm $n$ times, returns the arm (e.g. $\cA$) with highest empirical average $\hat{r}^\cA_n$, if 
$\hat{r}^\cA_n- \hat{r}^\cB_n \geq \sqrt{\frac{4\sigma^2}{n}\log\left(\frac{\log^2 (n)}{\delta}\right)} \: .$

In the statement of our theorems, the usual Landau notation ${\scriptscriptstyle \mathcal{O}_\delta}(1)$ stands for any function whose limit is, as $\delta$ goes to 0, equal to 0. Similarly, we will use $\widetilde{\delta} \leq 23\delta\log(\frac{1}{\delta})$ instead of $\delta$ for the sake of clarity. Exact values of $\widetilde{\delta}$, precise statements and proofs are mostly delayed to Appendix \ref{appendix:UCB_vs_ETC}.

The performances of ETC are now well understood.
\begin{proposition}[\cite{perchet2013multi}]\label{thm:ETC_for_two_arms}
With probability greater than $1-\widetilde{\delta}$, ETC returns the best arm at a stage $\tau_d \leq T_\delta$ where
$$T_\delta \leq \frac{32\sigma^2}{\Delta^2}\log\left(\frac{1}{\delta}\right)\Big(1+{\scriptscriptstyle \mathcal{O}_\delta}(1)\Big) \: .
$$
So the regret of ETC at the time of decision verifies
$$
R(\tau_d) \leq  \frac{16\sigma^2}{\Delta}\log\left(\frac{1}{\delta}\right)\Big(1+{\scriptscriptstyle \mathcal{O}_\delta}(1)\Big) \: .$$
\end{proposition}
On the other hand, the seminal algorithm which is ``optimal'' for regret minimization is called UCB; it sequentially pulls the arm with the highest ``score'' (the sum of the empirical average plus some error term) while its decision rule is not satisfied. Although its regret is small, it is not guaranteed that the best arm will be identified in a short time.
\begin{proposition}[\cite{auer2002finite}]
With probability at least $1-\widetilde{\delta}$, the regret of UCB is bounded as
\begin{align*}R(\tau_d) &\leq \frac{8\sigma^2}{\Delta}\log\left(\frac{1}{\delta}\right)\Big(1+{\scriptscriptstyle \mathcal{O}_\delta}(1)\Big)\, .
\end{align*}
There is no guarantee that $\tau_d$ is uniformly bounded.
\end{proposition}

Our family of algorithms  interpolates between UCB and ETC by introducing a single parameter $\alpha \in[1,+\infty]$ whose extreme values correspond respectively to focusing solely on regret minimization (i.e., our algorithm identifies with UCB) or best-arm identification (it identifies with ETC). To each value of  $\alpha \in [1,+\infty]$ corresponds a different trade-off: Indeed, the bigger the  $\alpha$, the bigger the regret but the smaller the decision time.



More precisely, we introduce and  study a continuum of algorithms UCB$_{\alpha}$. For $n\in\N^*$, define $\varepsilon_{n} = \sqrt{\frac{2\sigma^2}{n}\log(\frac{3\log^2 n}{\delta})}$.  For $\alpha,\delta \in \R^+$, UCB$_{\alpha}$ allocates the user at the population with the highest score
$ \argmax_i \hat{r}^i_{n^i} + \alpha \varepsilon_{n^i}$. It returns an arm $i\in\{A,B\}$ if it dominates the other arm $j$ in the sense that $\hat{r}^i_{n^i} - \varepsilon_{n^i} \geq \hat{r}^j_{n^j} + \varepsilon_{n^j}$.

\begin{algorithm}
\caption{UCB$_{\alpha}$}
\label{Algo:UCB_alpha}
\hspace*{\algorithmicindent} \textbf{Input:} $\alpha,\delta$
\begin{algorithmic} [1]
  \Repeat { over \texttt{n}}
 \For{\texttt{each population $i \in \{\cA, \cB\}$}}
    \State{$\varepsilon^i_n = \sqrt{\frac{2\sigma^2}{n^i}\log(\frac{3\log^2 n^i}{\delta})}$}
    \EndFor
  
  \State{\texttt{Assign next user to population} \par \hskip\algorithmicindent ~~~~~$i_n =\argmax_{i \in \{\cA,\cB\}} \hat{r}^i_n + \alpha \varepsilon^i_n$}
  \State{$i^* = \argmax_{i \in \{\cA,\cB\}} \hat{r}^i_n$}
  \Until{$\hat{r}^{i^*}_n - \varepsilon^{i^*}_n > \hat{r}^j_n + \varepsilon^j_n$ for $j\neq i$}\\
  \Return{$i^*$}
\end{algorithmic}
\end{algorithm}

By construction, the UCB$_{\alpha}$  algorithm will keep both indexes around the same level. But due to the factor $\alpha>1$, the intervals of decision with width $\varepsilon_{n^\cA}$ and $\varepsilon_{n^\cB}$ (without the factor $\alpha$ that is only used in the sampling policy, not in the decision rule) will eventually become disjoint. Thus, while behaving like a UCB-type algorithm to minimize regret, UCB$_{\alpha}$ can still return the identity of the best arm.
The following theorem makes precise this trade-off between time of decision and regret.  
\begin{theorem}\label{thm:UCB_alpha_delta}
With probability greater than $1-\widetilde{\delta}$, UCB$_{\alpha}$ has a regret $R(\tau_d)$ at its time of decision satisfying
$$
R(\tau_d) \leq \Big(\frac{8\sigma^2}{\Delta}c_\alpha +\Delta\Big)  \log\big(\frac{1}{\delta}\big)\Big(1+{\scriptscriptstyle \mathcal{O}_\delta}(1)\Big) \: .
$$
where the constant $c_\alpha \in [1,7]$ is defined by 
$$c_\alpha = \min\big\{\frac{(\alpha+1)^2}{4}, \frac{4\alpha^2}{(\alpha-1)^2} \big\} \text{ and } c_1=1\ ; \ c_\infty = 4\: .$$

On that event, the time of decision  satisfies $\tau_d \leq T_\delta$ with
\begin{align*}
T_\delta \leq \: & \frac{\alpha^2+1}{(\alpha-1)^2} \left(\frac{16\sigma^2}{\Delta^2}c_\alpha+ 1 \right) \log\big(\frac{1}{\delta}\big)\Big(1+{\scriptscriptstyle \mathcal{O}_\delta}(1)\Big)  \: .
\end{align*}
\end{theorem}
When $\alpha$ goes to 1, the leading term of the regret goes to $8\sigma^2\log(\frac{1}{\delta})/\Delta$ but the time of decision becomes infinite. When $\alpha\to\infty$, the leading term becomes $32\sigma^2\log(\frac{1}{\delta})/\Delta$ and the time of decision is of order $64\sigma^2\log(\frac{1}{\delta})/\Delta^2$. The fact that the constant $c_\alpha$ is not monotonic  in $\alpha$, is an artifact of the proof, i.e., a byproduct of two different analyses of the same problem (for either small or large values of $\alpha$).  Indeed, Figure \ref{fig:2_arms_UCB_vs_ETC} actually indicates that regret of UCB$_{\alpha}$ is certainly monotonic with respect to $\alpha$ as expected.

ETC has a regret bounded by an expression of order $\frac{16\sigma^2}{\Delta}\log\frac{1}{\delta}$. This is twice bigger than the regret of UCB$_{\alpha}$ for small $\alpha$. These are only one-sided bounds and do not allow to conclude on which algorithm gets a lower regret, but experiments show that UCB$_{\alpha}$ for small $\alpha$ indeed presents an advantage in terms of regret versus ETC, at the cost of a higher decision time. See Figure \ref{fig:2_arms_UCB_vs_ETC}.

When $\alpha$ goes to infinity, the exploration term is dominant for UCB$_{\alpha}$, and it will always pull the least pulled arm. As a consequence, it becomes  a variant of ETC with a sub-optimal decision criterion. We denote it by ETC'. It allocates to the  populations $\cA$ and $\cB$ uniformly and selects  $\cA$ if
$\hat{r}^\cA_n{-} \hat{r}^\cB_n \geq \sqrt{\frac{8\sigma^2}{n}\log\left(\frac{3\log^2 (n)}{\delta}\right)}$. Its confidence interval width is $\sqrt{2}$ times larger than the one of ETC because of the different concentration arguments used in designing the algorithms: ETC uses a concentration lemma on the difference $\hat{r}^\cA_n - \hat{r}^\cB_n$, while UCB$_{\alpha}$ (and its limit ETC') deals separately with arm $\cA$ and $\cB$ since their number of samples might be different. Note that the difference between ETC and ETC' is a specificity of the two-armed bandit case, as the generalization of ETC to more than two arms \cite{perchet2013multi} considers per-arm confidence intervals.

\begin{figure}[ht]
\begin{center}\includegraphics[scale=0.15]{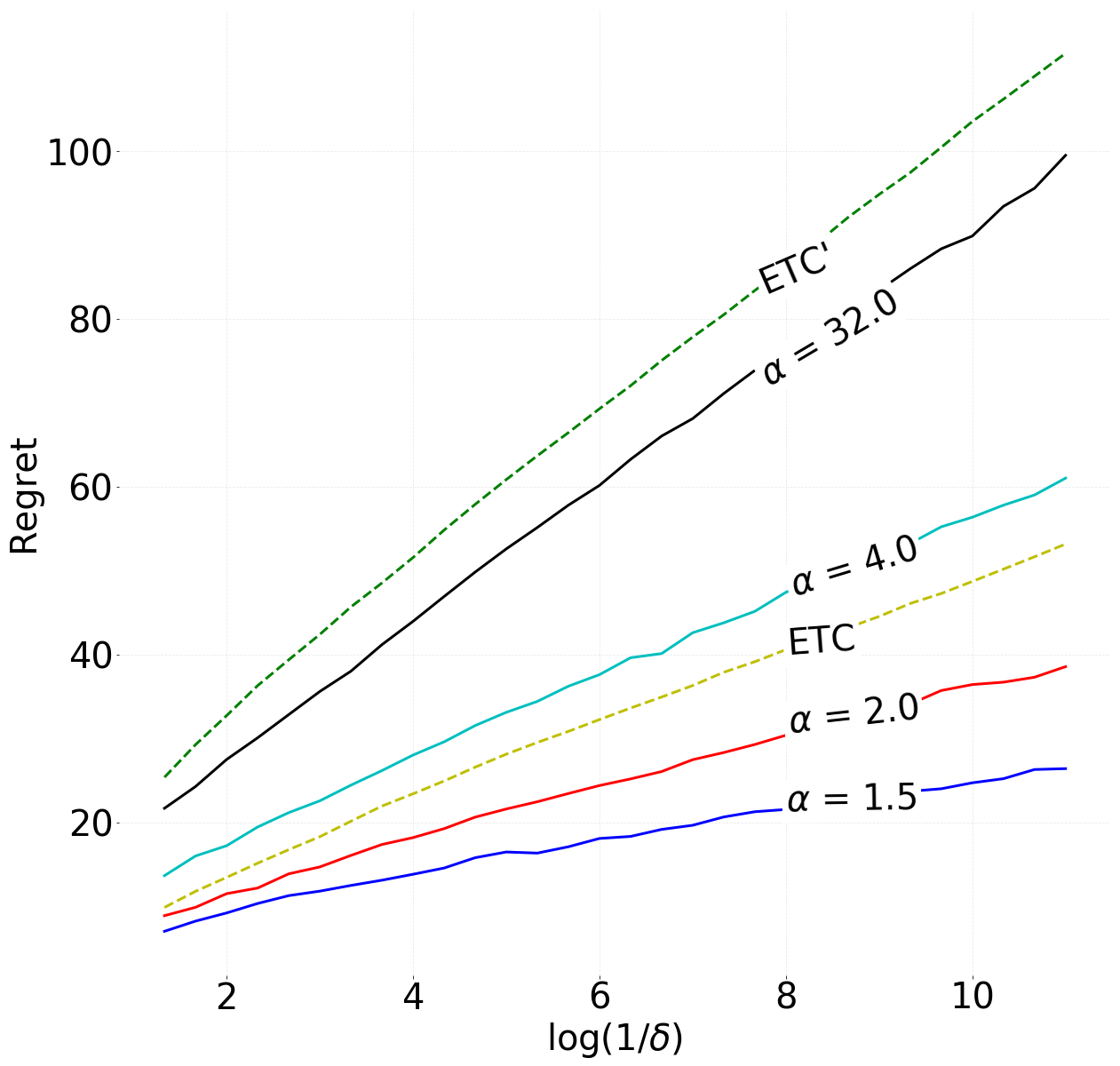} \\
\includegraphics[scale=0.15]{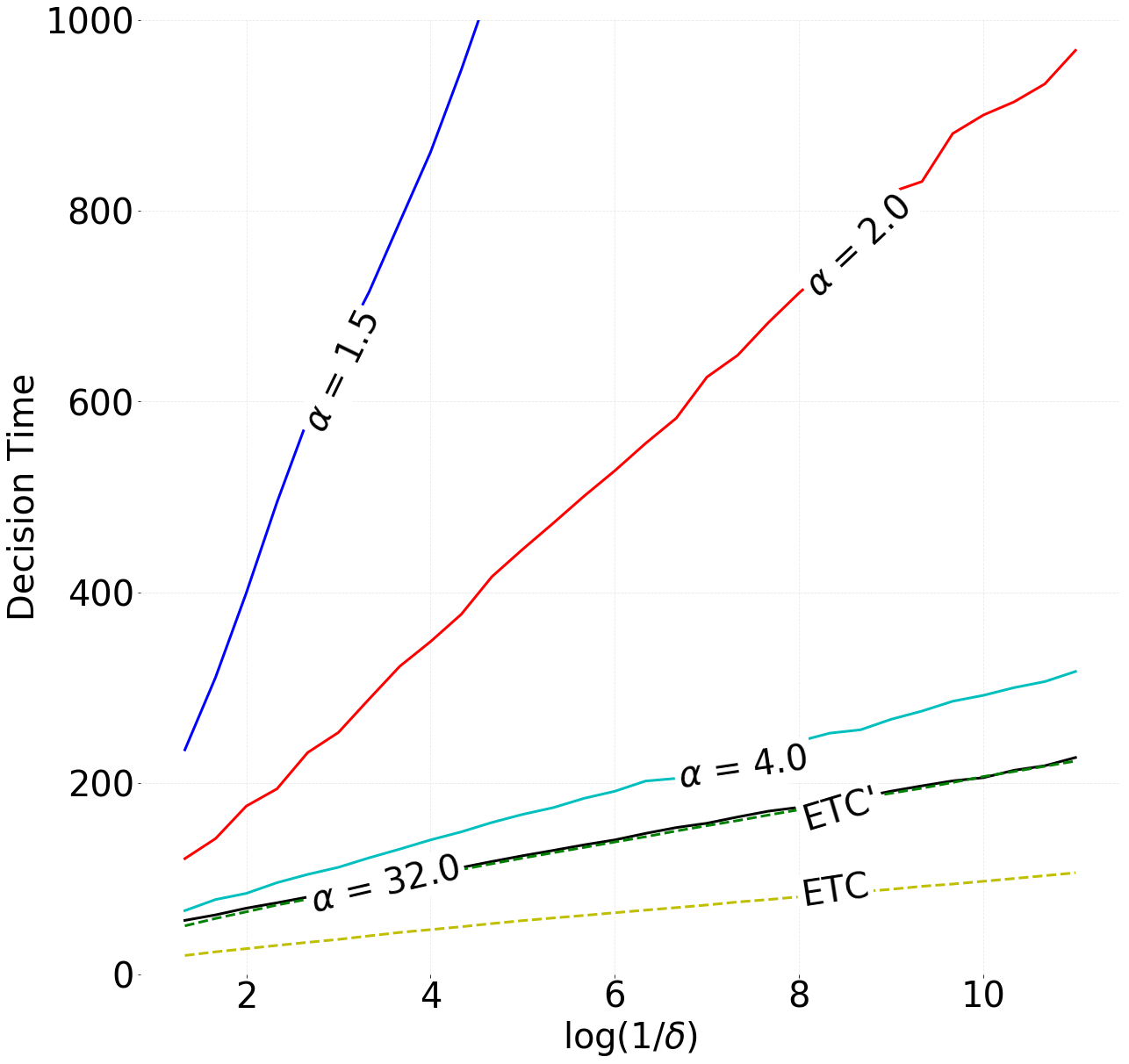}
\end{center}
\caption{\textbf{Comparison of ETC and UCB$_{\alpha}$. Left: regret at selection. Right: time of selection.} The curves average 1000 experiments with two Gaussian arms of means 0 and 1 and variance 1.}
\label{fig:2_arms_UCB_vs_ETC}
\end{figure}

On the other hand, UCB$_{\alpha}$ generalizes immediately when the number of populations is greater than $2$. In that case, we assume that the index of the optimal population is  $1$ and we denote by $\Delta_k=r^1 -r^k$ the gap associated to the subobtimal population $k$. Then we can derive the following corollary from Theorem \ref{thm:UCB_alpha_delta}.
\begin{corollary}
With $K$ different arms, and  with probability greater than $1-\widetilde{\delta}$, UCB$_{\alpha}$ has a decision time smaller than
\begin{align*}
\left( \frac{(\alpha+1)^2}{(\alpha-1)^2} \left(\frac{8\sigma^2}{\Delta^2_{\min
}}c_\alpha+ 1 \right) + \sum_{k=2}^K \frac{8\sigma^2}{\Delta^2_k}c_\alpha +  K \right) \\
\times\log\big(\frac{K}{\delta}\big)\Big(1+{\scriptscriptstyle \mathcal{O}_\delta}(1)\Big)   \end{align*}
and its regret $R(\tau_d)$  satisfies
$$
R(\tau_d) \leq \sum_{k=2}^K\Big(\frac{8\sigma^2}{\Delta_k}c_\alpha +\Delta\Big)  \log\big(\frac{K}{\delta}\big)\Big(1+{\scriptscriptstyle \mathcal{O}_\delta}(1)\Big) \: .
$$
\end{corollary}
With $K>2$ population, ETC would stop sampling population $k$ as soon as there exists another population $i$ such that
$\hat{r}^i_n{-} \hat{r}^k_n \geq 4\sqrt{\frac{\sigma^2}{n}\log\left(\frac{3K\log^2 (n)}{\delta}\right)}$. As a consequence, its decision time will be upper-bounded by
$$
\Big(\frac{32\sigma^2}{\Delta^2_{\min}}+\sum_{k=2}^K \frac{32\sigma^2}{\Delta^2_k}\Big)\log(\frac{K}{\delta})(1+{\scriptscriptstyle \mathcal{O}_\delta}(1))
$$
and its regret smaller than
$$
\
\Big(\sum_{k=2}^K \frac{32\sigma^2}{\Delta^2}\log(\frac{K}{\delta})\Big)(1+{\scriptscriptstyle \mathcal{O}_\delta}(1))\ .
$$
As a consequence, even if UCB$_\alpha$ interpolates between UCB and ETC in term of regret, it actually outperforms ETC in terms of decision time when $\alpha \to \infty$.

\subsection{How does inflating the exploration term lead to a finite decision time? A proof sketch}\label{sec:proof_sketch}

We consider an algorithm which pulls $\arg\max_{i\in\{\cA,\cB\}} \hat{r}^i_{n^i} + \alpha \varepsilon_{n^i}$ with $\alpha>1$, and stops if $|r^\cA_{n^\cA} - r^\cB_{n^\cB}| > \varepsilon_{n^\cA} + \varepsilon_{n^\cB}$ and returns the arm with highest mean at that point.

Recall that $\varepsilon_n$ is a quantity  close to $1/\sqrt{n}$, up to logarithmic terms in $n$ and multiplicative constants, hence $1/\varepsilon^2_{n} \approx c n$ for some constant $c$. If we can prove that as long as no decision is taken, $1/\varepsilon^2_{n^\cA}$ and $1/\varepsilon^2_{n^\cB}$ are bounded from above, then we obtain an upper bound on the time $t=n_\cA+n_\cB$ before a decision is taken.

Suppose that the best arm is $\cA$. Bounding $n_\cB$ is done through classical bandit arguments: since $\cB$ is the worse arm, a UCB-type algorithm does not pull it often. The challenge is to show that our algorithm also controls $n_\cA$.

 We can make use of a concentration result of the form: with probability $1-\tilde{\delta}$, $\hat{r}^\cA(n^\cA) + \varepsilon_{n^\cA} \geq r^\cA$ and $\hat{r}^\cB(n^\cB) - \varepsilon_{n^\cB} \leq r^\cB$. The decision criterion ensures that if a decision is taken and these concentration inequalities hold, then the arm returned is the correct one. Indeed under this concentration event, $\hat{r}^\cB(n^\cB) - \hat{r}^\cA(n^\cA) \leq r^\cB - r^\cA + \varepsilon_{n^\cB} + \varepsilon_{n^\cA}$ , which is strictly smaller than $\varepsilon_{n^\cB} + \varepsilon_{n^\cA}$ . Hence $\cB$ cannot be returned: if the algorithm stops, it is correct.

The algorithm will keep both indexes roughly equal, hence $\hat{r}^\cA_{n^\cA} + \alpha \varepsilon_{n^\cA} \approx \hat{r}^\cB_{n^\cB} + \alpha \varepsilon_{n^\cB}$ and the ''pulling'' confidence intervals with width $\alpha\varepsilon_n$ will never get disjoint. But as $\varepsilon_{n^\cA}$ and $\varepsilon_{n^\cB}$ get small, the smaller ''decision'' confidence intervals with width $\varepsilon_n$ will eventually separate, as seen in Figure~\ref{fig:confidence_intervals}.

More formally, if $\cA$ is pulled and no decision was taken yet, then the index of $\cA$ is big, $\hat{r}^\cA(n^\cA) + \alpha\varepsilon_{n^\cA} > \hat{r}^\cB(n^\cB) + \alpha\varepsilon_{n^\cB}$, but not so big that the algorithm stops, i.e. $\hat{r}^\cA(n^\cA) - \varepsilon_{n^\cA} \leq \hat{r}^\cB(n^\cB) + \varepsilon_{n^\cB}$. By combining the two, we obtain the relation $n_\cA \propto \frac{1}{\varepsilon^2_{n^\cA}} \leq \frac{(\alpha+1)^2}{(\alpha-1)^2}\frac{1}{\varepsilon^2_{n^B}} \propto \frac{(\alpha+1)^2}{(\alpha-1)^2} n^\cB$ , where $\propto$ is to be read as the informal statement that the quantities are roughly proportional.

To sum-up the idea of the proof: $\cB$ will not be pulled much since it is the worst arm and we employ an UCB-type algorithm. $\cA$ will be pulled less than a factor depending on $\alpha$ times $\cB$. Thus as long as no decision is taken, $t = n^\cA+n^\cB$ is bounded, by a quantity $T_\delta$. We conclude that the decision time is smaller than $T_\delta$. 

\begin{figure}[ht]
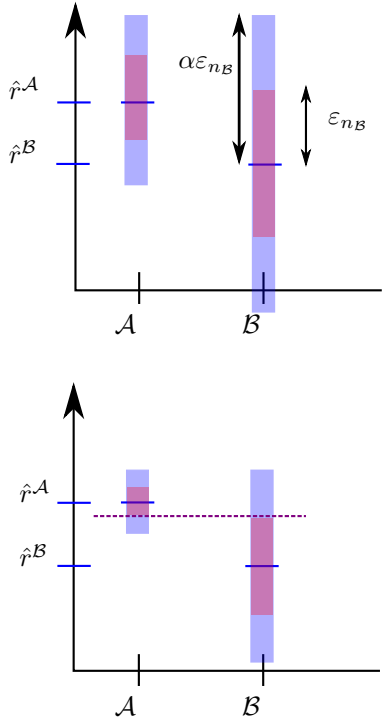
\caption{\textbf{Confidence intervals used in UCB$_{\alpha}$}. Top: Before the time of decision, the indexes $\hat{r}(n) + \alpha\varepsilon_n$ are aligned, the tighter $\varepsilon_n$ intervals also overlap. Bottom: time of selection: the tight $\varepsilon_n$ confidence intervals become disjoint.}
\label{fig:confidence_intervals}
\end{figure}


\section{Extensions to non-iid settings}
\label{extension_Non_IID}
When an internet platform wants to A/B test two versions of its website, purchases from customers that are buying recurrently can not be considered as independent. When a technical A/B test is run, it is also usual to split servers in two populations to A/B test the business impact of the latency of a new code version. In this case, observations from the same server are not independent. This setting is not restricted to online marketing. Clinical trials classically measure the survival time or quality of life of patients over time and adaptive testing is a key challenge in this context, however it is so far either heuristic \cite{Press2009} or under the restrictive assumption to observe the rewards before the next decision \cite{Berry1995, Hu_Rosenberger_2006}. More recently, \cite{Negoescu2018} studied a similar setting of reward arriving over time, however with a different objective, namely finding an adaptive way to stop the test for a patient taking into account a cost of testing.

The setting of multi-armed bandit presented in the previous section has been applied to A/B tests \cite{kaufmann2014, Yang2017, Zhao2016} with independent rewards. The goal of this section is to show that we can extend algorithms presented in the first section in a more complex setting and provide a framework to practitioners for interpolating between best arm identification and regret minimization in other settings than the iid case.

To model theses aspects, we show that decisions of the multi-armed bandit algorithm can be taken at a unit level (e.g. users, patients, servers...). Once allocated to a population $\cA$ or $\cB$, a unit $u$  interacts with the system during the whole A/B test . When time increases, the system gathers more and more signal on a unit arrived early in the A/B test. We will also assume in order to be able to take a causal decision at the end of the A/B test that units already exposed to one treatment can not be switched from population. A unit stays in the same population during the whole A/B test. Intuitively, the system will estimate the performance of one technology by averaging its performance on the different units.

\subsection{Notations}

We need to differentiate the units (e.g. users) randomly assigned to populations $\cA$ and $\cB$ from their associated rewards. In the iid setting, the reward $r^\cA_u$ associated to a unit $u$ in population $\cA$ was assumed to be observed instantly. Now, we assume that, for each unit $u$ we have been seeing so far, we observe noisy version of this reward over time and $r^\cA_u$ is only the unknown expectation of this process. Population-specific notation is symmetric, thus, for the sake of readability, we only detail notations for the control $\cA$ and assume the corresponding one for the treatment $\cB$.

More formally, we assume the units $u$ are i.i.d. samples from an unknown distribution and arriving in the test dynamically over time. To each unit $u$ is associated an unknown reward $r^\cA_u$ which is an i.i.d. sub-gaussian r.v. with expected value $r^\cA$ and variance $\sigma_r^2$, as well as an arrival time $\mathcal{T}_u$. $\cA(t)$ denotes the set of units in $\cA$ of cardinality $n^\cA_t$ at time $t$. Then, given a unit $u$ and starting at time $\mathcal{T}_u$, we observe over time random outcomes $r^\cA_{u,t} = r^\cA_u + \varepsilon_{u,t}$ where $\varepsilon_{u,t}$ is a zero-mean sub-gaussian variable with variance $\sigma_\varepsilon^2$. At time $t$, the unit $u$ has generated $t - \mathcal{T}_u + 1$ outcomes $r^\cA_{u,t_u} \cdots r^\cA_{u,t}$.

At each time step, the algorithm has access to all the rewards generated by all the users already present in the A/B test and the precision on $r^\cA_u$ will increase with time as more samples $r^\cA_{u,t}$ are gathered. 

\begin{figure}[t]
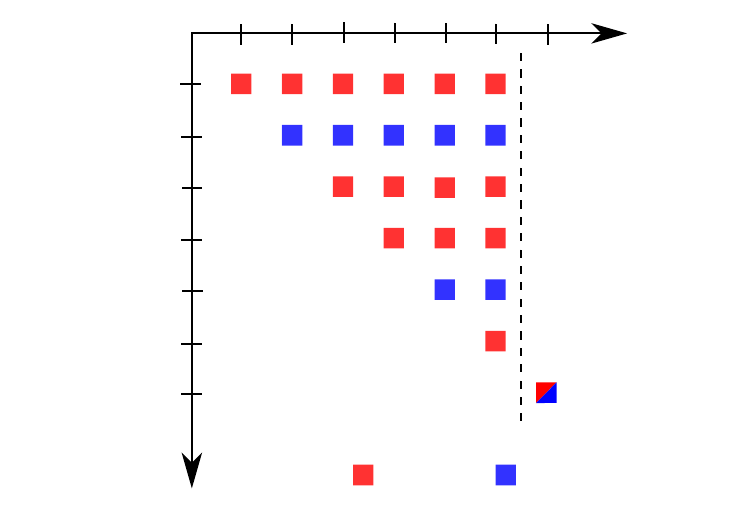
\caption{\textbf{The unit allocation problem}. At each stage $t\in\N$, a new unit is allocated to either population $\cA$ or $\cB$ and a sample from every unit arrived before $t$ is observed.}
\label{fig:units}
\end{figure}

In this setting, a natural estimator to consider is 
\begin{equation*}
\hat{r}_t^{\cA}:= \frac{1}{n_t^{\cA}} \sum_{u \in \cA(t)} \frac{1}{t-\mathcal{T}_u+1} \sum_{s=\mathcal{T}_u}^{t} r_{u,s}^{\cA}
\end{equation*}
We call this estimator the \textit{Mean of means} estimator. With this estimator, we can design algorithms that can reach a trade-off between regret minimization and best arm identification and see how this regret depend on $\sigma_r^2$ and $\sigma_\varepsilon^2$.
We could have considered other estimators such that the \textit{Total Mean} estimator defined as 
\begin{equation*}
\hat{R}_t^{\cA}:= \frac{1}{ \sum_{u \in \cA(t)} t-\mathcal{T}_u+1} \sum_{u \in \cA(t)}\sum_{s=\mathcal{T}_u}^{t} r_{u,s}^{\cA}
\end{equation*}
which is also an unbiased estimator of $r$. Yet unfortunately, this estimator puts more weights on older units than on the more recent one. In the case where all units have more or less the same noise, this is not really an issue (but this is more or less the only case where the total mean estimator has good behavior). Unfortunately, with adaptive sampling algorithm (such as UCB), then it could be the case that a new unit is allocated to a population after an exponential long time. In that case, the different weights put on different units can be of different order of magnitude, preventing fast convergence of the estimate to the estimated mean.

A second motivation for choosing the \textit{Mean of means} estimator is because it opens doors for generalizing to more complex models on the stochastic processes underlying the units behavior. Indeed, it can be seen as an average over the units of the per-unit stochastic process expected values. Here, we assume the random outcomes $r^i_{u,t}$ are i.i.d. with expectation $r^i_u$ (technically our results hold for martingale difference sequence) . Then the technical part shows how to combine concentration results on each of the units to derive bandit algorithms with good properties. With a different model on the per-user random outcomes (as. mean reverting process or cyclic process), our proof techniques could be used as long as it is possible to construct an estimator of the expectation that concentrate well enough.

Precise statements and proofs of results presented in this section can be found in Appendix \ref{appendix:UCBMM_vs_ETCMM}.

\subsection{ETC}


The ETC algorithm allocates units alternatively to $\cA$ and $\cB$, choosing possibly the first population at random. To simplify the analysis, we are actually going to assume that 2 units arrive at each stage, one of each being allocated to each population. So that if  ETC stops  at stage $t\in\N$, then both populations have $n=t$ units, and regret is $n\Delta$. 

The stopping rule criterion of ETC is simply the following:
$$|\hat{r}^\cA_{n} - \hat{r}^\cB_{n}| {>} \sqrt{\frac{4\Big(\sigma_r^2 {+} \frac{\sigma_\varepsilon^2\log(en)}{n}\Big)}{n}\log\Big(\frac{\pi^2n^2}{3\delta}\Big)}$$

\begin{theorem}
With probability at least $1-\delta$, the ETC algorithm outputs the best population and stops after having allocating a total of at most $\tau_d \leq T_\delta$ units, where.
\begin{align*}
T_\delta = \frac{32\sigma_r^2}{\Delta^2}\log(\frac{1}{\delta})(1 {+} {\scriptscriptstyle \mathcal{O}_\delta}(1)) {+} \frac{\sigma^2_\varepsilon}{\sigma^2_r} \log\log(\frac{1}{\delta})(1{+}{\scriptscriptstyle \mathcal{O}_\delta}(1))\: .
\end{align*}

Its regret at $\tau_\delta$ is equal to $\frac{\Delta}{2}\tau_\delta$.
\end{theorem}

\subsection{UCB-MM}

The variant of UCB we consider is defined with respect to the following index of performance of population $i \in \{\cA,\cB\}$ defined by
\begin{equation}\label{EQ:DefUCBMM}
 \hat{r}_t^{i} +  \sqrt{\frac{2\Big(\sigma_r^2 + \frac{\sigma_\varepsilon^2\log(en^i_t)}{n^i_t}\Big)}{n^i_t}\log\Big(\frac{(4n^i_t)^4}{2\delta}\max\{1,\frac{n^i_t\sigma_r^2}{\sigma_\varepsilon^2}\}\Big)} \end{equation}

Using this index,  UCB-MM  is described in Algorithm \ref{Algo:UCBMM}. We mention here that having random numbers of units with random numbers of occurrences prevents us for deriving the more or less standard concentration inequalities derived for the other algorithm. Indeed, it actually requires to combine several types of different inequalities, which explains the non-standard term in the $\sqrt{\log(\cdot)}$ part of the index.
\begin{algorithm}
\caption{UCB-MM$_{\alpha}$}
\label{Algo:UCBMM}
\hspace*{\algorithmicindent} \textbf{Input:} $\alpha,\delta$
\begin{algorithmic} [1]
  \Repeat { over \texttt{t}}
 \For{\texttt{each population $i \in \{\cA, \cB\}$}}
    \State{$\varepsilon^i_t = \sqrt{\frac{2\Big(\sigma_r^2 + \frac{\sigma_\varepsilon^2\log(en^i_t)}{n^i_t}\Big)}{n^i_t}}$}
    \State{$\qquad\qquad\times\sqrt{\log\Big(\frac{(4n^i_t)^4}{\delta}\max\{1,\frac{n^i_t\sigma_r^2}{\sigma_\varepsilon^2}\}\Big)}$}
    \EndFor
  
  \State{\texttt{Assign next user to population} \par \hskip\algorithmicindent ~~~~~$i_t =\argmax_{i \in \{\cA,\cB\}} \hat{r}^i_t + \alpha \varepsilon^i_t$}
  \State{$i^* = \argmax_{i \in \{\cA,\cB\}} \hat{r}^i_t$}
  \Until{$\hat{r}^{i^*}_t - \varepsilon^{i^*}_t > \hat{r}^j_t + \varepsilon^j_t$ for $j\neq i$}\\
  \Return{$i^*$}
\end{algorithmic}
\end{algorithm}

\begin{theorem}
With probability at least $1-\widetilde{\delta}$, the regret of UCB-MM is bounded at stage $t_d$ as
\begin{align*}R(\tau_d) &\leq \frac{8\sigma^2_r}{\Delta}\log(\frac{1}{\delta})(1+{\scriptscriptstyle \mathcal{O}_\delta}(1))\\&\hspace{1cm}+\frac{\sigma_r^2}{\sigma_\varepsilon^2}\Delta\log\log(\frac{1}{\delta})(1+{\scriptscriptstyle \mathcal{O}_\delta}(1))\, .
\end{align*}
There is no guarantee that $\tau_d$ is uniformly bounded.
\end{theorem}
We emphasize here that the dependency of the total regret with respect to the noise variance $\sigma_\varepsilon^2$ is negligible compared to its dependency with respect to the variance of the unit performance $\sigma_r^2$. Indeed, regret has a  $\log \frac{1}{\delta}$ factor in front of $\sigma_r^2$ (a term which is unavoidable, even without extra noise, i.e., if $\sigma_\varepsilon^2=0$). On the other hand, the multiplicative factor of $\sigma_\varepsilon^2$ is only  double logarithmic, in  $\log\log \frac{1}{\delta}$. Moreover, the additional number of units required to find the best population is, asymptotically, independent of $\Delta$,  the proximity measure of the two populations.

In the index definition of UCB-MM (Equation \eqref{EQ:DefUCBMM}), letting $\sigma_\varepsilon^2$ goes to 0 gives a void bound (the index is basically always $+\infty$). This is an artefact of the proof needed for having only an extra $\log\log(\cdot)$ term. It is also possible to use the following alternative error term for UCB-MM
$$ \sqrt{\frac{2\sigma_r^2}{n^i_t}\log\Big(\frac{9\log^2(n_t^i)}{\delta}\Big)}+ \sqrt{\frac{2\sigma_\varepsilon^2\log(en^i_t)}{(n^i_t)^2}\log\Big(\frac{(4n^i_t)^4}{\delta}\Big)}
$$
This error term converges, as $\sigma_\varepsilon^2$ goes to 0, to the usual error term of UCB. Unfortunately, the regret dependency in $\sigma_\varepsilon^2$ deteriorates as it  scales with 
$$
\Big(\frac{8\sigma^2_r}{\Delta}\log \frac{1}{\delta} + \frac{\sigma_r^2}{\sigma_\varepsilon^2} \Delta \sqrt{\log \frac{1}{\delta}}\Big)(1+{\scriptscriptstyle \mathcal{O}_\delta}(1))
$$ 
thus with a $\sqrt{\log(\cdot)}$  extra term instead of a $\log\log(\cdot)$ one.

\subsubsection{Interpolating between regret minimization and best arm identification}

As in the iid case, it is possible, with unit, to define UCB-MM$_{\alpha}$ to interpolate between UCB-MM and ETC-MM by multiplying the error term of UCB-MM by a factor $\alpha \in [1,+\infty]$.
\begin{theorem}
With probability at least $1-\widetilde{\delta}$, the regret of UCB-MM$_{\alpha}$ is bounded at stage $T$ as
\begin{align*}R(\tau_d) &\leq \Big(\frac{8\sigma^2_r}{\Delta}c_\alpha+ \Delta\Big)\log(\frac{1}{\delta})(1+{\scriptscriptstyle \mathcal{O}_\delta}(1))
\\&\hspace{1cm}+\frac{\sigma_r^2}{\sigma_\varepsilon^2}\Delta\log\log(\frac{1}{\delta})(1+{\scriptscriptstyle \mathcal{O}_\delta}(1))\, .
\end{align*}
Moreover, the time of decision is upper-bounded by
\begin{align*}
\tau_d \leq \: & \frac{\alpha^2+1}{(\alpha-1)^2} \left(\frac{16\sigma_r^2}{\Delta}c_\alpha+ 1 \right) \log\big(\frac{1}{\delta}\big)\Big(1+{\scriptscriptstyle \mathcal{O}_\delta}(1)\Big)  \\
&\hspace{1cm}+2\frac{\sigma_r^2}{\sigma_\varepsilon^2}\Delta\log\log(\frac{1}{\delta})(1+{\scriptscriptstyle \mathcal{O}_\delta}(1))  \: .
\end{align*}
\end{theorem}

The proof of this result proceeds as in the iid case (see the sketch of proof in section \ref{sec:proof_sketch}), assuming that we can provide a suitable concentration inequality to bound the deviation of the mean-of-means estimator. The main difference is the concentration arguments used. The analysis is adaptive in that it gives a bound on the decision time for any model, as long as we are able to provide concentration inequalities for the population means. The concentration inequality is used to obtain a bound on the number of allocations to the sub-optimal population, then the inflated confidence intervals mechanically provide a bound on the number of pulls of the best population with respect to the sub-optimal one.

As in the iid case, we have stated our results for $K=2$ populations, but it's straightforward to generalize them for $K>2$ different populations.

\textbf{Practical remark.}
In practice, attributing users dynamically to populations could be hard to handle in production (for instance, the population of the user must be stored to interact with him when he is coming back on the platform..). This is why we also provide another anytime bound in appendix in a simpler setting where we assume that all the units are already present at the beginning of the test, thus allowing to allocate them to populations using a simple hash on their identifiers. Based on this bound, the practitioner can stop the test as early as possible such that the decision is statically sufficient. However, this bound can not help to do a tradeoff between regret minimization and best arm identification. 

On the other hand, we assume that the reward and noise were subGaussian random variable, with known variance (proxy) $\sigma_r^2$ and $\sigma_\varepsilon^2$. Our results and techniques can be generalized to the case where the random variables $r_t^i$ and $\varepsilon_{n,t}$ are bounded (say, in $[0,1]$) with unknown variance. One just need to use empirical Bernstein concentration inequalities \cite{Maurer2009}.


\section{Experiments}
\label{expe}
On a simple iid setting, we show the performance proved in Section \ref{SectionBothWorlds} on Figure \ref{fig:2_arms_UCB_vs_ETC}. There are two Gaussian arms with same variance 1 and means 0 and 1. The two graphs on the figure show the decision time and the regret at the time of decision of several algorithms for a range of values of $\log(1/\delta)$. The algorithms shown are ETC, four instances of UCB$_{\alpha}$ for alpha in $(1.5, 2, 4, 32)$ and ETC', the variant of ETC to which UCB$_{\alpha}$ tends to when $\alpha\to\infty$. We highlight a few conclusions from these plots, which are all in agreement with the theoretical results.
\begin{itemize}
\item The algorithm with lowest decision time is ETC, the algorithm with lowest regret is UCB$_\alpha$ with small $\alpha$.
\item For $\alpha\geq 4$, UCB$_\alpha$ has lower regret and higher decision time than ETC', but is worse than ETC on both criteria.
\item For $\alpha$ equal to 1.5 or 2, UCB$_\alpha$ has lower regret and higher decision time than ETC.
\end{itemize}
UCB$_\alpha$ is seen empirically to realize a trade-off between its two limiting algorithms UCB and ETC', and there is an interval for $\alpha$ in which UCB$_\alpha$ is a trade-off between UCB and ETC. The numerical relations between the bounds can also be verified: the regret of UCB$_{1.5}$ is almost twice smaller than the regret of ETC, which is twice smaller than the regret of ETC'.

We then show on Figure \ref{fig:emp_results} empirical results on the unit setting presented in Section \ref{extension_Non_IID}. At each time step, a user arrives and then generate a reward according to $r_u$ for every time step until the end of the game. In our simulation, $r_{u}$ is sampled from a normal distribution with mean 0 for population $\mathcal{A}$ (respectively with mean 1 for population $\mathcal{B}$)  and variance $\sigma_r^{2}$ equal to 1. The noise $\epsilon$ at each time step is also Gaussian of mean 0 and variance $\sigma_{\epsilon}^2$ equal to 1. The data can be seen as a triangular matrix (cf Fig. \ref{fig:units}). We compare performances between UCB$_{\alpha}$ for different values of $\alpha$ and ETC' in terms of regret and times of decision. We see that we are able to reproduce what we observed in the iid setting in the unit setting. We observe again a factor 4 between the regret of ETC' and UCB. With UCB$_{\alpha}$, we can realize a trade-off between ETC' and UCB, both in terms of regret and decision time. 

\begin{figure}[ht]
\center
\begin{center}
\includegraphics[scale=0.15]{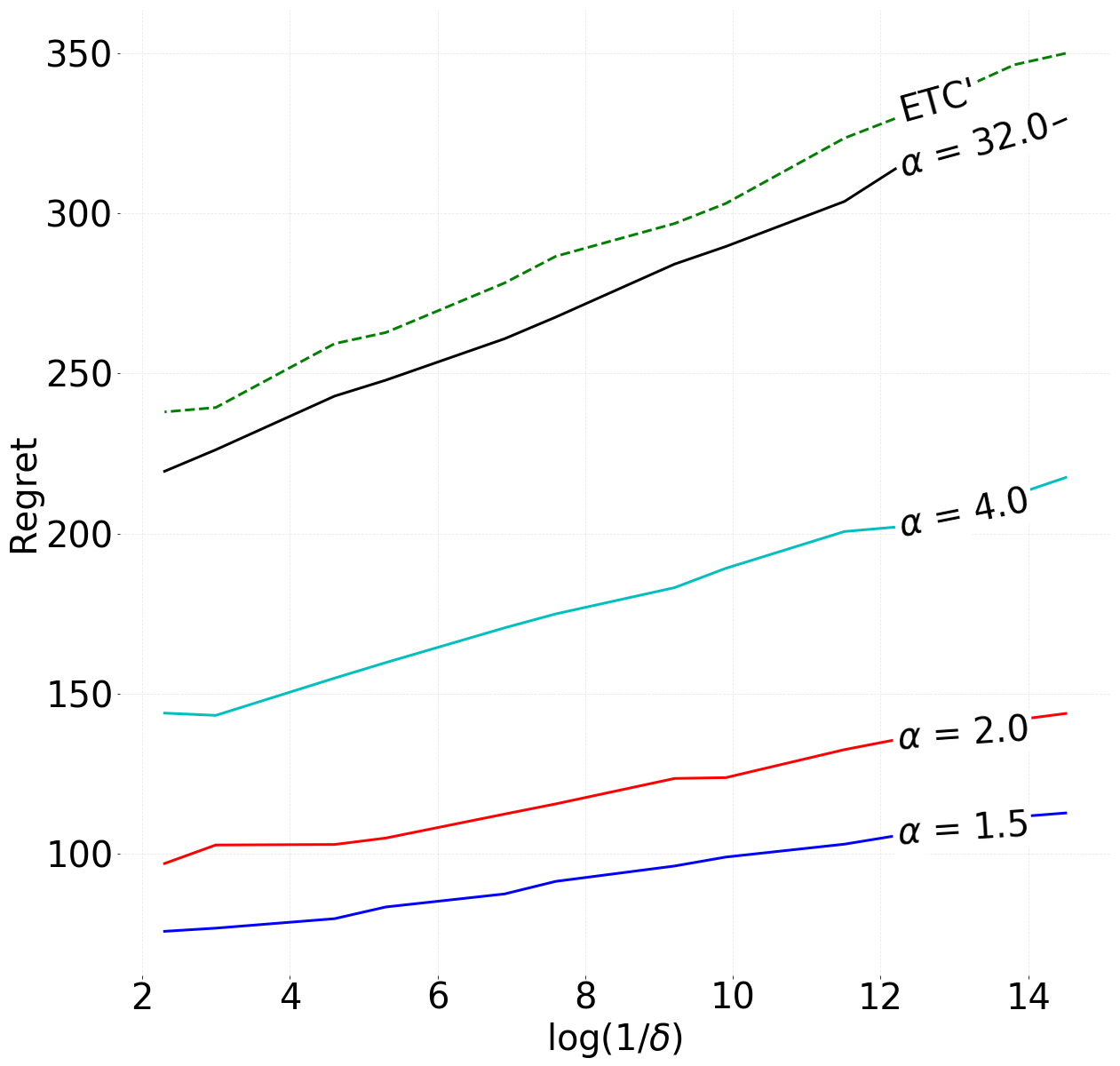} 
\includegraphics[scale=0.15]{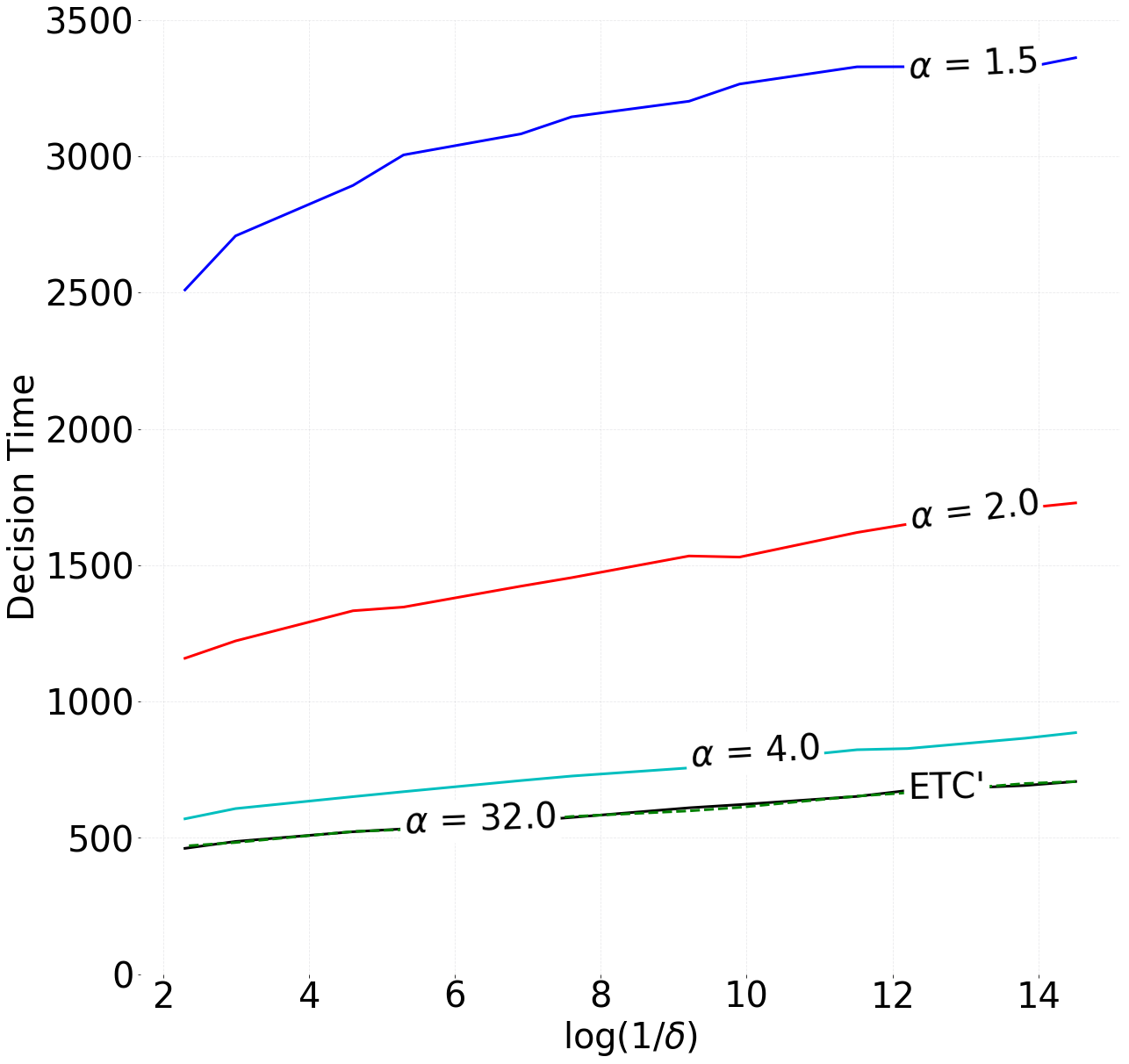} 
\end{center}
\caption{\textbf{First: Regret in the unit setting. Second: Decision time for the unit setting.} The curves are averages of 1000 experiments with two Gaussian arms with means 0 and 1 and $\sigma_r = 1$, $\sigma_\epsilon = 1$.}
\label{fig:emp_results}
\end{figure}


\section{Conclusion}


We studied A/B tests in the fixed confidence setting from the two perspectives of regret minimization and best arms identification. We introduced a class of algorithms that optimizes at the same time both objectives. It interpolates between optimal algorithms designed for each case. Our study also shed light on an effect often seen in practice: the UCB algorithm not only minimizes regret but also identifies the best arm in finite time, as soon as its exploration term is slightly inflated. This inflation is nearly always present in practice. Indeed, for UCB to be a valid algorithm for a noise and the confidence interval not to be bigger than necessary, it would need to be run with a constant matching exactly the unknown sub-Gaussian norm of the noise.

We extended our study to a non-iid setting by deriving adapted concentration results for the mean-of-means estimator, again obtaining algorithms which interpolate between the two objectives. We would however like to warn practitioners on an intensive use of bandit algorithms for A/B tests. Even if data are collected through time, it can prove to be difficult to define a time of arrival for units not correlated with the data to ensure the iid assumption holds as bandit algorithms that neglect this aspect could behave quite poorly. Handling such problem would require to model the unit stochastic process \textit{conditionally} to the time of arrival which would be a use case for generalizing our results to more complex stochastic processes for unit modelling.

 \section{Aknowledgement}
 The fourth author has benefited from the support of the FMJH Program Gaspard Monge in optimization and operations research (supported in part by EDF), from the Labex LMH and from the CNRS through the PEPS program.
\bibliography{main_final}
\bibliographystyle{abbrvnat}

\newpage

\onecolumn

\appendix


%

\section{Concentration Inequalities}\label{Appendix:Concentration}


We provide here required results of anytime bounds that hold in the i.i.d. setting for the interested reader. 

\paragraph{Hoeffding on intervals}
\begin{lemma}
Let $Z_t$ be a $\sigma^2$-subGaussian martingale difference sequence  then for every $\delta >0$ and every integers $T_1 \leq  T_2  \in \lN$
$$
\mathbb{P} \Big\{ \exists t\in[T_1;T_2], \overline{Z}_t \geq \sqrt{\frac{2\sigma^2}{t}  \log\left(\frac{1}{\delta}\right) \phi\left(\frac{T_2}{T_1}\right)}\Big\} \leq \delta
$$
where the mapping $\phi(\cdot)$ is defined by
$$
\phi(x) = \frac{1+x+2\sqrt{x}}{4\sqrt{x}}$$  
 and it holds that $1- \frac{(x-1)^2}{16} \leq \frac{1}{\phi(x)} \leq 1$.
\label{lem:hoeffding_interval}
\end{lemma}
\begin{proof}
For any $\lambda > 0$, it holds that 
\begin{align*}
\lE\left[e^{\lambda t\overline{Z}_t}\middle| \cF_{t-1}\right]
     = \lE\left[e^{\lambda Z_t} \middle| \cF_{t-1}\right] e^{\lambda (t-1)\overline{Z}_{t-1}} \leq e^{\frac{\lambda^2 \sigma^2}{2}} e^{\lambda (t-1)\overline{Z}_{t-1}}\end{align*}
    since $Z_t$ is $\sigma^2$-subGaussian. 
Thus $X_t = e^{\lambda t\overline{Z}_t - \frac{\lambda^2\sigma^2 t}{2}}$ is a supermartingale and we derive from Markov's inequality that:
\begin{align}
\lP\left(\exists t \in [T_1,T_2], \lambda t\overline{Z}_t \geq \frac{\lambda^2\sigma^2}{2}t + \varepsilon \right)
\leq
e^{-\varepsilon}
\label{eq:hoeffding_interval_proof_1}
 \end{align}
As $t \mapsto \sqrt{t}$ is concave on $[T_1,T_2]$, we can find $\beta$ such that $\beta \sqrt{t} \geq \frac{\lambda^2\sigma^2}{2}t + \varepsilon$ with equality in $T_1$ and $T_2$. From these two equalities, we choose $\lambda$ and $\beta$ as function of $\varepsilon$:
\begin{align*}
    \lambda = \sqrt{\frac{2\varepsilon}{\sigma^2\sqrt{T_1 T_2}}} ~~~~~~~~~ \beta = \varepsilon \left(\frac{\sqrt{T_1}+\sqrt{T_2}}{\sqrt{T_1 T_2}}\right)
\end{align*}
Combining this with \eqref{eq:hoeffding_interval_proof_1}, we finally obtain:
\begin{align*}
   & \lP\left(\exists t \in [T_1,T_2],
    \lambda t\overline{Z}_t \geq \beta \sqrt{t} \right)
    =
    \lP\left(\!
    \exists t \in [T_1,T_2],
    \overline{Z}_t \geq \sqrt{\frac{\sigma^2\varepsilon}{t}\frac{(\sqrt{T_1}+\sqrt{T_2})^2}{4\sqrt{T_1 T_2}}} 
    \right)
    \leq
    e^{-\varepsilon},
\end{align*}
this yields the result.
\end{proof}

\paragraph{Anytime Hoeffding}

\begin{lemma}\label{lem:anytime_empirical_hoeffding}
Let $Z_t$ be a $\sigma^2$-sub-Gaussian martingale difference sequence. Then, for any $\alpha>0 $ and $\delta>0$ satisfying
$$ 1+\frac{1}{\log(\frac{\log(2)}{\delta})}<\alpha < \frac{\log(\frac{\log(2)}{\delta})}{8},$$
it holds that
\begin{align*}
\mathbb{P} \Big\{ \exists t\in \N, \overline{Z}_t \geq  \sqrt{\frac{2\sigma^2}{t}\log(\frac{\log_2^\alpha(t)}{\delta})} \Big\} \leq c_{\alpha} \delta\sqrt{\log\Big(\frac{\delta}{\log(2)}\Big)} + \delta(\frac{e^{\frac{\alpha}{2}}}{\log(2)}+1) \: ,
\end{align*}
\begin{align*}\text{where }\  c_\alpha &=  \frac{e^{\frac{\alpha}{2}}}{\log(2)} \sqrt{\frac{1}{8\alpha}}\frac{1}{\alpha(1-\frac{\alpha}{2\log(\frac{\log(2)}{\delta}})-1}
 \leq     \frac{e^{\frac{\alpha}{2}}}{\log^\alpha(2)} \sqrt{\frac{1}{8\alpha}}\frac{16/15}{\alpha-16/15},\end{align*}
so that $c_2\approx 1.62$ and $\log_2(\cdot)$ is the natural logarithm in basis 2, with the extra assumption that $\log_2(1)=1$.

If $\alpha, \delta>0$ are such that
$$
1+\frac{1}{\log(\frac{1}{\delta})} \leq \alpha \leq \frac{1}{2}\log(\frac{1}{\delta})
$$
then
\begin{align*}
\mathbb{P} \Big\{ \exists t\in \N, \overline{Z}_t \geq  \sqrt{\frac{2\sigma^2}{t}\log(\frac{\log_2^\alpha(t)}{\delta})} \Big\}\leq c'_{\alpha} \delta \log^{\frac{\alpha}{2}}(\frac{1}{\delta}) + \delta(1+\sqrt{\frac{8\alpha}{\log(1/\delta)}}) \: ,
\end{align*}
$$\text{where }\ c'_{\alpha} =  \Big(\frac{2e}{\alpha}\Big)^{\alpha/2} \zeta\big(\alpha-\frac{\alpha^2}{2\log(1/\delta)}\big) \leq     \Big(\frac{2e}{\alpha}\Big)^{\alpha/2} \zeta\big(\frac{3\alpha}{4}),$$
so that $ c'_2\approx 7.11$. 
\end{lemma}
\begin{proof}
We are going to use the fact that with probability at least $1-\delta$,  for all $s \in [T_1,T_2]$,
$$
s\Big(\overline{Z}_s - \EE[Z]\Big) \leq \sqrt{2\sigma^2 s\phi(\frac{T_2}{T_1})\log(\frac{1}{\delta})}\: .
$$
Define $\varepsilon_t=  \sqrt{\frac{2\sigma^2}{t}\log(\frac{\log_2^\alpha(t)}{\delta})} $ so that with $\gamma=1+\eta>1$,

$$
\PP\Big\{ \exists t \in \N, \overline{Z}_t \geq \varepsilon_t\Big\} \leq \sum_{m=0}^{ \infty}\mathbb{P} \Big\{ \exists t\in [\gamma^m, \gamma^{m+1}], \overline{Z}_t  \geq \varepsilon_t\Big\}
$$
The case $m=0$ (corresponding to $t=1 $) will be handled separatly at the cost of an extra $\delta$ term. Note that
\begin{align*}
& \sum_{m= \lfloor \frac{1}{\log_2(\gamma)}\rfloor}\mathbb{P} \Big\{ \exists t\in [\gamma^m, \gamma^{m+1}], \overline{Z}_t  \geq \sqrt{\frac{2\sigma^2}{t}\log(\frac{\log_2^\alpha(\gamma^m)}{\delta})} \Big\}\\
&\leq \sum_{m= \lfloor \frac{1}{\log_2(\gamma)}\rfloor}\mathbb{P} \Big\{ \exists t\in [\gamma^m, \gamma^{m+1}], \overline{Z}_t  \geq  \sqrt{\frac{2\sigma^2}{t}\log(\frac{\log_2^\alpha(\gamma^m)}{\delta})} \sqrt{\phi(\gamma)(1-\frac{\eta^2}{16})}\Big\}\\
&=    \sum_{m= \lfloor \frac{1}{\log_2(\gamma)}\rfloor}\mathbb{P} \Big\{ \exists t\in [\gamma^m, \gamma^{m+1}], \overline{Z}_t  \geq
\sqrt{\frac{2\sigma^2}{t} \phi(\frac{\gamma^{m+1}}{\gamma^m}) \log \left(\left(\frac{\log_2^{\alpha}(\gamma^m)}{\delta}\right)^{1-\frac{\eta^2}{16}}\right)} \Big\} \: .
\end{align*}
We can now apply Lemma~\ref{lem:hoeffding_interval}. And this gives, assuming $\gamma \leq 2$ (i.e., $\eta <1$) for the moment, 
\begin{align*}
 \PP \Big\{ \exists t \in \N, \overline{Z}_t \geq \varepsilon_t\Big\}
&\leq \sum_{m= \lfloor \frac{1}{\log_2(\gamma)}\rfloor} \Big(\frac{\delta}{ \log_2^\alpha(\gamma^{m})}\Big)^{1-\frac{\eta^2}{16}} + \delta\\
&=    \Big(\frac{\delta}{\log_2^\alpha(\gamma)}\Big)^{1-\frac{\eta^2}{16}}\sum_{m= \lfloor \frac{1}{\log_2(\gamma)}\rfloor} \frac{1}{m^{\alpha(1-\frac{\eta^2}{16})}}+ \delta\\
& \leq (\frac{\delta}{\log^\alpha(2)})^{1-\frac{\eta^2}{16}}\Big(1+\frac{1}{\alpha(1-\frac{\eta^2}{16})-1}\frac{1}{\log_2(\gamma)}\Big()+\delta\Big)\\
& \leq (\frac{\delta}{\log^\alpha(2)})^{1-\frac{\eta^2}{16}}\Big(1+\frac{1}{\alpha(1-\frac{\eta^2}{16})-1}\frac{1}{\eta}\Big)+\delta
\end{align*}


The choice of $\eta^2 = 8 \alpha /\log(\log(2)/\delta)$, which ensure that $\eta < 1$ as long as $\alpha < \frac{1}{8}\log(\frac{\log(2)}{\delta})$ gives
\begin{align*}
 \PP \Big\{ \exists t \in \N, \overline{Z}_t \geq \varepsilon_t\Big\}
&\leq \frac{\delta}{\log^\alpha(2)} e^{\frac{\alpha}{2}}\Big(\sqrt{\frac{\log(\frac{\log(2)}{\delta})}{8\alpha}}\frac{1}{\alpha(1-\frac{\alpha}{2\log(\frac{\log(2)}{\delta}})-1}+1\Big)+\delta\\
& \leq \delta\sqrt{\log\big(\frac{\log(2)}{\delta}\big)}\frac{1}{\log^\alpha(2)} e^{\frac{\alpha}{2}}\sqrt{\frac{1}{8\alpha}}\frac{16/15}{\alpha-16/15}+\delta\Big(\frac{e^{\frac{\alpha}{2}}}{\log^\alpha(2)}+1\Big)
\end{align*}

We now consider the case where $\gamma$ might be bigger than  $2$, but let us assume for now that $\gamma <5$ (i.e., $\eta \leq 4$) and the exact same argument with the choice of $\eta^2 = 8 \alpha /\log(1/\delta)$ gives

\begin{align*}
\PP\Big\{ \exists t \in \N, \overline{Z}_t \geq \varepsilon_t\Big\} &\leq \sum_{m=1}^{ \infty} \Big(\frac{\delta}{ \log_2^\alpha(\gamma^{m})}\Big)^{1-\frac{\eta^2}{16}} + \gamma\delta\\
&\leq \Big(\frac{\delta}{\log_2^\alpha(1+\eta)}\Big)^{1-\frac{\eta^2}{16}} \zeta(\alpha(1-\frac{\eta^2}{16}))+ \gamma\delta\\
&\leq \delta^{1-\frac{\eta^2}{16}}\Big(\frac{4}{\eta}\Big)^{\alpha} \zeta(\alpha(1-\frac{\eta^2}{16}))+ \gamma\delta\\
&=    \delta e^{\frac{\alpha}{2}} \Big(\frac{2\log(1/\delta)}{\alpha}\Big)^{\alpha/2} \zeta\big(\alpha-\frac{\alpha^2}{2\log(1/\delta)}\big)+\Big(1+\sqrt{\frac{8\alpha}{\log(1/\delta)}}\Big)\delta\\
&=    \delta \log^{\frac{\alpha}{2}}\big(\frac{1}{\delta}\big)  \Big(\frac{2e}{\alpha}\Big)^{\alpha/2} \zeta\big(\alpha-\frac{\alpha^2}{2\log(1/\delta)}\big)+\Big(1+\sqrt{\frac{8\alpha}{\log(1/\delta)}}\Big)\delta
\end{align*}
\end{proof}

\begin{corollary}
Let $Z_t$ be a $\sigma^2$-sub-Gaussian martingale difference sequence. Then for $\delta>0$ small enough
\begin{align*}
\mathbb{P} \Big\{ \exists t\in \N, \overline{Z}_t \geq  \sqrt{\frac{2\sigma^2}{t}\log(\frac{\log^2(t)}{\delta})} \Big\}
\leq \widetilde{\delta} \: ,
\end{align*}
where $\widetilde{\delta}$ is defined by either
$$
\widetilde{\delta}=c_2 \frac{{\delta}}{\log^2(2)}\sqrt{\log(\frac{\log^3(2)}{{\delta}})}+\frac{{\delta}}{\log^2(2)}(\frac{e}{\log(2)}+1) \ ,
$$
or, depending on the range of $\delta$, by, 
$$
\widetilde{\delta}=c'_2\frac{{\delta}}{\log^2(2)} \log(\frac{\log^2(2)}{{\delta}})+5\frac{{\delta}}{\log^2(2)}.
$$
In the first case, $\widetilde{\delta}$ is of the order of $\delta\sqrt{\log(\frac{1}{\delta})}$ and, in the second case, of the order of $\delta\log(\frac{1}{\delta})$.
\end{corollary}

The following concentration inequality will be useful for the non-iid case. In that framework, at some stage $t$, the total number of users in the population is random, denoted by $n_t$. We recall that $r_n$ denotes the mean of the $n$-th user of the population and $\varepsilon_{n,s}$  is the random white noise for that user after he is in the population for $s$ stages. In the remaining, we assume that the expectation of $r_u$ is equal to $r$ and that this random variable is $\sigma_r^2$-subGaussian, On the other hand, the expectation of $\varepsilon$ is naturally $0$ and this random variable is $\sigma_\varepsilon^2$-subGaussian. An algorithm is therefore a sampling policy $\mathcal{A}$  that indicates after seeing the first $n$ values of $r_u$ plus some  empirical average noise $\overline{\varepsilon}_{u,t}$ 
 at time $t \in \N$ whether to add a new user or not. We denote by $\mathcal{A}_t \in \{0,1\}$ the decision to include a new user or not at stage $t$. We denote by $\mathcal{T}_n \in \N$ the time where the $n$-th user is added, by $\tau_n=\mathcal{T}_{n+1}-\mathcal{T}_{n}$ the number of stages with exactly $n$ users and by $\tau_{m:n} = \sum_{s=m}^{n-1} \tau_s$ the number of stages between the arrival of the $m$-th user and the $n$-th one. We also denote by $n+t$ the number of user at stage $t \in \N$,

\begin{proposition}
For any algorithm and $\delta >0$, it holds
\begin{align*}
\mathds{P}\Big\{\exists t \in \N, \frac{1}{n_t} \sum_{u=1}^{n_t} r_u + \overline{\varepsilon}_{u,t-\mathcal{T}_u+1} 
\leq r - \sqrt{\frac{2\Big(\sigma_r^2 + \frac{\sigma_\varepsilon^2\log(en_t)}{n_t}\Big)}{n_t}\log\Big(\frac{(4n_t)^4}{{6\delta}}\max\{1,\frac{n_t\sigma_r^2}{\sigma_\varepsilon^2}\}\Big) } \Big\} \leq \widetilde{\delta}
\end{align*}
\end{proposition}

\begin{proof}We rewrite the statement of the proposition and notice that we just need to prove that, for any $n \in \N$, 
\begin{equation*}\mathds{P}\Big\{ \exists 1 \leq s \leq \tau_n, \overline{r}_n+ \frac{\overline{\varepsilon}_{1,s+\tau_{1:n}}+\ldots+\overline{\varepsilon}_{n,s}}{n} \leq r- \sqrt{\frac{2\Big(\sigma_r^2 + \frac{\sigma_\varepsilon^2\log(en)}{n}\Big)}{n}\log\Big(\frac{36n^4}{{\delta}}\max\{1,\frac{n\sigma_r^2}{\sigma_\varepsilon^2}\}\Big) } \Big\} \leq   \frac{1}{3}\frac{\widetilde{\delta}}{n^{3/2}}\ .\end{equation*}
The exponent $3/2$ will come from the fact that $\widetilde{\delta}$ is of the order of $\delta \log \frac{1}{\delta}$ (and not $\delta$). We will even prove the following
\begin{equation*}\mathds{P}\Big\{ \exists 1 \leq s < \infty, \overline{r}_n+ \frac{\overline{\varepsilon}_{1,s+\tau_{1:n}}+\ldots+\overline{\varepsilon}_{n,s}}{n} \leq r - \sqrt{\frac{2\Big(\sigma_r^2 + \frac{\sigma_\varepsilon^2\log(en)}{n}\Big)}{n}\log\Big(\frac{36n^4}{{\delta}}\max\{1,\frac{n\sigma_r^2}{\sigma_\varepsilon^2}\}\Big) } \Big\} \leq   \frac{1}{3}\frac{\widetilde{\delta}}{n^{3/2}}\end{equation*}
We will decompose the considered event defined on $\{1 \leq s < \infty\}$ in two, depending whether $ s \leq 6n^2 \max\{1,\frac{n\sigma_r^2}{\sigma_\varepsilon^2}\}$ or $ s > 6n^2 \max\{1,\frac{n\sigma_r^2}{\sigma_\varepsilon^2}\}$.
 
 In the first case,  we aim at proving that for all $s \leq 6n^2 \max\{1,\frac{n\sigma_r^2}{\sigma_\varepsilon^2}\}$,
\begin{equation*}
\mathds{P}\Big\{\overline{r}_n+ \frac{\overline{\varepsilon}_{1,s+\tau_{1:n}}+\ldots+\overline{\varepsilon}_{n,s}}{n} \leq r- \sqrt{\frac{2\Big(\sigma_r^2 + \frac{\sigma_\varepsilon^2\log(en)}{n}\Big)}{n}\log\Big(\frac{36n^4}{{\delta}}\max\{1,\frac{n\sigma_r^2}{\sigma_\varepsilon^2}\}\Big) } \Big\} \leq  \frac{1}{18}\frac{\widetilde{\delta}}{2n^{7/2}}\frac{1}{\max\{1,\frac{n\sigma_r^2}{\sigma_\varepsilon^2}\}}.
\end{equation*}

As usual, take $\eta >0$ and let us try to upper bound

\begin{equation*}\mathds{P}\Big\{r_1+\ldots+r_n-nr+\overline{\varepsilon}_{1,s+\tau_{1:n}}+\ldots+\overline{\varepsilon}_{n,s} \geq \eta \Big\}\end{equation*} 
Introduce $\lambda >0$ and Markov inequality yields
\begin{equation*}\mathds{P}\Big\{r_1+\ldots+r_n-nr+\overline{\varepsilon}_{1,s+\tau_{1:n}}+\ldots+\overline{\varepsilon}_{n,s} \geq \eta \Big\} \leq \mathds{E} [ e^{\lambda(r_1+\ldots+r_n-nr+\overline{\varepsilon}_{1,s+\tau_{1:n}}+\ldots+\overline{\varepsilon}_{n,s} )} ] e^{-\lambda \eta}
\end{equation*} 
Note that the realizations of $\overline{\varepsilon}^n_s$ and $\mu^n$ are independent of the other average once we have taken the decision to include that user, i.e., conditionally to $\mathcal{T}_n$
\begin{align*}
\mathds{E} [ e^{\lambda(r_1+\ldots+r_n-nr+\overline{\varepsilon}_{1,s+\tau_{1:n}}+\ldots+\overline{\varepsilon}_{n,s} )} ]  &=  \mathds{E}\Big[  \mathds{E}\big[ e^{\lambda (r_n-r+\overline{\varepsilon}_{n,s})}e^{\lambda(r_1+\ldots+r_{n-1}-(n-1)r+\overline{\varepsilon}_{1,s+\tau_{1:n}}+\ldots+\overline{\varepsilon}_{n-1,s+\tau_{n-1:n}})}\big| \mathcal{T}_n\big] \Big] \\
 & = \mathds{E}\Big[  \mathds{E}\big[ e^{\lambda (r_n-r+\overline{\varepsilon}_{n,s})}\big| \mathcal{T}_n\big]\mathds{E}\big[ e^{\lambda(r_1+\ldots+r_{n-1}+\overline{\varepsilon}_{1,s+\tau_{1:n}}+\ldots+\overline{\varepsilon}_{n-1,s+\tau_{n-1:n}})}\big| \mathcal{T}_n\big] \Big] \\
&  = \mathds{E}\Big[  \mathds{E}\big[ e^{(r_n-r+\overline{\varepsilon}_{n,s})}\big]\mathds{E}\big[ e^{\lambda(r_1+\ldots+r_{n-1}-(n-1)r+\overline{\varepsilon}_{1,s+\tau_{1:n}}+\ldots+\overline{\varepsilon}_{n-1,s+\tau_{n-1:n}})}\big| \mathcal{T}_n\big] \Big]\\
 & = \mathds{E}\big[ e^{(r_n-r+\overline{\varepsilon}_{n,s})}\big] \mathds{E}\big[   e^{\lambda(r_1+\ldots+r_{n-1}-(n-1)r+\overline{\varepsilon}_{1,s+\tau_{1:n}}+\ldots+\overline{\varepsilon}_{n-1,s+\tau_{n-1:n}})}\big]
\end{align*}
Let us rewrite, for clarity, the last expectation of the r.h.s. as
\begin{equation*}
\mathds{E}\big[   e^{\lambda(r_1+\ldots+r_{n-1}-(n-1)r+\overline{\varepsilon}_{1,s+\tau_{1:n}}+\ldots+\overline{\varepsilon}_{n-1,s+\tau_{n-1:n}})}\big] = \mathds{E}\big[   e^{\lambda(r_1+\ldots+r_{n-1}-(n-1)r+\overline{\varepsilon}_{1,s+\tau_{n-1}+\tau_{1:n-1}}+\ldots+\overline{\varepsilon}_{n-1,s+\tau_{n-1}})}\big]
\end{equation*}
If we condition similarly to $\mathcal{T}_{n-1}$, we can focus on upper-bounding
\begin{align*}
\mathds{E}\Big[e^{\lambda (r_{n-1}-\mu+  \overline{\varepsilon}_{n-1,s+\tau_{n-1}})} \Big] &= \sum_{j=s+1}^\infty\mathds{E}\Big[e^{\lambda (r_{n-1}-\mu+ \overline{\varepsilon}_{n-1,j})}\mathds{1}\{s+\tau_{n-1} = j\} \Big]\\
&\leq  \sum_{j=s+1}^\infty e^{\lambda^2(\frac{\sigma_r^2}{2}+\frac{\sigma_{\varepsilon}^2}{2j})} \mathds{E}\Big[\mathds{1}\{s+\tau_{n-1} = j\} \Big] \leq e^{\lambda^2(\frac{\sigma_r^2}{2}+\frac{\sigma_{\varepsilon}^2}{2(s+1)})}
\end{align*}
Choosing $\lambda = \frac{\eta}{n\sigma_r^2+ \sigma_{\varepsilon}^2\log(en)}$ gives that
\begin{equation*}\mathds{P}\Big\{r_1+\ldots+r_n-nr+\overline{\varepsilon}_{1,s+\tau_{1:n}}+\ldots+\overline{\varepsilon}_{n,s} \geq \eta \Big\} \leq  e^{-\frac{\eta^2}{2\big(n\sigma_r^2+\sigma_{\varepsilon}^2\log(en)\big)}}\end{equation*} 
So, for all $n \in \N$ and all stage $s \leq  6n^2 \max\{1,\frac{n\sigma_r^2}{\sigma_\varepsilon^2}\}$ and with probability at least $1-\frac{\delta}{2}$, it holds

$$\overline{r}_n+ \frac{\overline{\varepsilon}_{1,s+\tau_{1:n}}+\ldots+\overline{\varepsilon}_{n,s}}{n} \geq \mu- \sqrt{\frac{2\Big(\sigma_r^2 + \frac{\sigma_\varepsilon^2\log(en)}{n}\Big)}{n}\log\Big(\frac{36n^4}{{\delta}}\max\{1,\frac{n\sigma_r^2}{\sigma_\varepsilon^2}\}\Big) }  $$

We now focus on the stages where $s > 6n^2 \max\{1,\frac{n\sigma_r^2}{\sigma_\varepsilon^2}\}$ . But first, notice that Lemma \ref{lem:anytime_empirical_hoeffding} implies that:
\begin{align*}
\mathds{P}\Big\{\exists n\in\N, \quad \overline{r}_n  \leq  r - \sqrt{2\sigma_r^2 \frac{\log\frac{12\log^2 (n)}{{\delta}}}{n}}\Big\} \leq \frac{\widetilde{\delta}}{4}\, ,
\end{align*}
and also, similarly,
\begin{align*}
\mathds{P}\Big\{\exists i\in\N, s \in \N, \quad \overline{\varepsilon}_{i,s}  \leq - \sqrt{2\sigma_\varepsilon^2 \frac{\log\frac{36 i^2\log^2 (s)}{{\delta}}}{s}}\Big\} \leq \frac{\widetilde\delta}{4}\, .
\end{align*}
This implies that, with probability at least $1-\frac{\widetilde\delta}{2}$, for every  $n \in \N$ and $s \geq \underline{s}:= 6n^2 \max\{1,\frac{n\sigma_r^2}{\sigma_\varepsilon^2}\}$
$$\overline{r}_n+ \frac{\overline{\varepsilon}_{1,s+\tau_{1:n}}+\ldots+\overline{\varepsilon}_{n,s}}{n} \geq r- \left( \sqrt{2\sigma_r^2 \frac{\log\frac{12\log^2 (n)}{{\delta}}}{n}} + \sqrt{\frac{2 \sigma_\varepsilon^2}{\underline{s}}\log\Big(\frac{36n^4}{{\delta}}\log^2(\underline{s})\Big)} \right)  $$
It only remains to notice that 
$$
 \sqrt{2\sigma_r^2 \frac{\log\frac{18\log^2 (n)}{{\delta}}}{n}} + \sqrt{\frac{2 \sigma_\varepsilon^2}{\underline{s}}\log\Big(\frac{36n^4}{{\delta}}\log^2(\underline{s})\Big)} \leq \sqrt{\frac{2\Big(\sigma_r^2 + \frac{\sigma_\varepsilon^2\log(en)}{n}\Big)}{n}\log\Big(\frac{36n^4}{{\delta}}\max\{1,\frac{n\sigma_r^2}{\sigma_\varepsilon^2}\}\Big)}
$$
This inequality is a consequence of the fact that $\sqrt{a}+\sqrt{\lambda b} \leq \sqrt{a+b}$ as soon as $\lambda \leq \frac{1}{6 \max\{1,a/b\}}$, and this gives the result.\end{proof}

\section{UCB$_{\alpha}$ and ETC: proofs.}\label{appendix:UCB_vs_ETC}

Our exact theorems are the two following.

\begin{theorem}\label{thm:UCB_alpha_delta_exact}
With probability greater than $1-\tilde{\delta}$, UCB$_{\alpha}$ with $\alpha>1$ returns the best arm at a stage $t_d$ with
\begin{align*}
\tau_d 
\leq \: (C_1+C_2)\log\frac{2}{\delta}+  2C_1\log\log(2C_1 \max\{\log\frac{2}{\delta}, \: 2\log(2C_1)\})
+ 2 C_2\log\log(2C_2 \max\{\log\frac{2}{\delta}, \: 2\log(2C_2)\}) \: ,
\end{align*}
with $C_1 = \frac{2\sigma^2}{\Delta^2}\min\{(\alpha+1)^2, \frac{16\alpha^2}{(\alpha-1)^2} \} + 1 $ and $C_2 = \frac{(\alpha+1)^2}{(\alpha-1)^2}C_1 + 1$.

With the same probability, the regret $R(\tau_d)$ of UCB$_{\alpha}$ at the time of decision verifies
\begin{align*}
R(\tau_d) &\leq \left( \frac{2\sigma^2}{\Delta}\min\{(\alpha+1)^2, \frac{16\alpha^2}{(\alpha-1)^2} \} + \Delta \right) \bigg(\log\frac{2}{\delta} + 2\log\log(2C_1 \max\{\log\frac{2}{\delta}, \: 2\log(2C_1)\})\bigg)
\end{align*}
\end{theorem}

\begin{theorem}\label{thm:ETC_for_two_arms_exact}
With probability greater than $1-\tilde{\delta}$, ETC returns the best arm at a stage $\tau_d$ with
\begin{align*}
\tau_d \leq \frac{32\sigma^2}{\Delta^2}\left(\log\frac{1}{\delta} + 2\log\log(\frac{32\sigma^2}{\Delta^2} \max\{\log\frac{1}{\delta}, \: 2\log(\frac{32\sigma^2}{\Delta^2})\})\right) \: .
\end{align*}
With the same probability, the regret $R_{ETC}$ of ETC at the time of decision verifies
\begin{align*}
R(\tau_d) \leq \frac{16\sigma^2}{\Delta} \left(\log\frac{1}{\delta} + 2\log\log(\frac{32\sigma^2}{\Delta^2} \max\{\log\frac{1}{\delta}, \: 2\log(\frac{32\sigma^2}{\Delta^2})\})\right) \:.
\end{align*}
\end{theorem}

We first prove a generic lemma, which will also be useful in the non-IID case.

\begin{lemma}\label{lem:generic_nb_pulls_bound}
Consider the two arms problem where $\cA$ is the best arm.

Let $\delta\in(0,1]$ and $(\varepsilon_{n})_{n\in\N}$ be a sequence such that with probability $1-\delta$, for all $n_\cA,n_\cB\in\N^*$ we have the concentration inequalities $\hat{r}^\cA(n_\cA) + \varepsilon_{n_\cA} \geq \mu_\cA$ and $\hat{r}^\cB(n_\cB) - \varepsilon_{n_\cB} \leq \mu_\cB$. Suppose that for all $n\geq n_0\in\N$, $\frac{1}{\varepsilon_{n+1}^2} - \frac{1}{\varepsilon_{n}^2} \leq C$.

Let an algorithm be such that it pulls each arm $n_0$ times, then pulls $\arg \max_{i\in(\cA,\cB)} \hat{r}^i(n_i) + \alpha \varepsilon_{n_i}$ for $\alpha>1$, and takes a decision if for some $i,j\in(\cA,\cB)$, $\hat{r}^i(n_i) - \varepsilon_{n_i} > \hat{r}^i(n_j) + \varepsilon_{n_j}$. Then with probability $1-\delta$, the algorithm correctly returns $\cA$ and at all times after $2n_0$ and prior to the decision
\begin{align*}
\frac{1}{\varepsilon^2_{n_\cB}} &\leq \frac{1}{\Delta^2}\min\{(\alpha+1)^2, \frac{16\alpha^2}{(\alpha-1)^2} \} + C \: ,\\
\frac{1}{\varepsilon^2_{n_\cA}} &\leq \frac{(\alpha+1)^2}{(\alpha-1)^2}\left( \frac{1}{\Delta^2}\min\{(\alpha+1)^2, \frac{16\alpha^2}{(\alpha-1)^2} \} + C \right) + C \: .
\end{align*}
\end{lemma}

\begin{proof}
We prove that the number of pulls of the best arm $\cA$ is bounded by a function of the number of pulls of $\cB$, which is itself bounded since it is the worse arm.

\paragraph{Relation between $\cA$ and $\cB$ when $\cA$ is pulled.}

\begin{itemize}
\item No decision was taken yet: $\hat{r}^\cA(n_\cA) - \varepsilon_{n_\cA} \leq \hat{r}^\cB(n_\cB) + \varepsilon_{n_\cB}$ (1),
\item $\cA$ is pulled: $\hat{r}^\cA(n_\cA) + \alpha\varepsilon_{n_\cA} \geq \hat{r}^\cB(n_\cB) + \alpha\varepsilon_{n_\cB}$ (2).
\end{itemize}
Subtract (1) from (2) to obtain $(\alpha+1)\varepsilon_{n_\cA} \geq (\alpha-1)\varepsilon_{n_\cB}$. Equivalently, $\frac{1}{\varepsilon^2_{n_\cA}}\leq \frac{(\alpha+1)^2}{(\alpha-1)^2}\frac{1}{\varepsilon^2_{n_\cB}} $ .
Since the left-hand-side grows only when $\cA$ is chosen and grows at most by $C$, we have that for all stages,
\begin{align*}
\frac{1}{\varepsilon^2_{n_\cA}} \leq \frac{(\alpha+1)^2}{(\alpha-1)^2}\frac{1}{\varepsilon^2_{n_\cB}} + C \: .
\end{align*}

\paragraph{Upper bound on $n_\cB$.}

When $\cB$ is pulled, $\hat{r}^\cA(n_\cA) + \alpha\varepsilon_{n_\cA} \leq \hat{r}^\cB(n_\cB) + \alpha\varepsilon_{n_\cB}$. With probability $1-\tilde{\delta}$, for all $n_\cA$ and $n_\cA$ we also have the concentration inequalities $\hat{r}^\cA(n_\cA) + \varepsilon_{n_\cA} \geq \mu_\cA$ and $\hat{r}^\cB(n_\cB) - \varepsilon_{n_\cB} \leq \mu_\cB$. Hence
\begin{align*}
\mu_\cA + (\alpha-1)\varepsilon_{n_\cA} &\leq \mu_\cB + (\alpha+1)\varepsilon_{n_\cB} \: .
\end{align*}
From this inequality we can get that $\frac{1}{\varepsilon^2_{n_\cB}}  \leq \frac{(\alpha+1)^2}{\Delta^2} $ .
Since the left-hand-side grows only when $\cB$ is pulled and grows at most by 1, we have for all stages
\begin{align*}
\frac{1}{\varepsilon^2_{n_\cB}}  &\leq \frac{(\alpha+1)^2}{\Delta^2} + C \: .
\end{align*}

In order to get another bound, relevant when $\alpha$ is big, we write that when no decision is taken and concentration holds, we have
\begin{align*}
\mu_\cA - 2\varepsilon_{n_\cA} \leq \hat{r}^\cA(n_\cA) - \varepsilon_{n_\cA} \leq \hat{r}^\cB(n_\cB) + \varepsilon_{n_\cB} \leq \mu_\cB + 2\varepsilon_{n_\cB}\: .
\end{align*}
This leads to $\varepsilon_{n_\cB} + \varepsilon_{n_\cA} \geq \frac{\Delta}{2}$. When $\cB$ is pulled, we also have the inequality $(\alpha+1)\varepsilon_{n_\cB} \geq (\alpha-1)\varepsilon_{n_\cA}$, such that
\begin{align*}
\frac{\Delta}{2} \leq \varepsilon_{n_\cB}(1 + \frac{\alpha+1}{\alpha-1}) = \varepsilon_{n_\cB} \frac{2\alpha}{\alpha-1} \: .
\end{align*}
This gives a second inequality for $n_B$. For all stages,
\begin{align*}
\frac{1}{\varepsilon^2_{n_\cB}} &\leq \frac{16 \alpha^2}{\Delta^2(\alpha-1)^2} + C \: .
\end{align*}

\paragraph{Bound on the decision time.}
We have the following inequalities for all stages prior to the decision,
\begin{align*}
\frac{1}{\varepsilon^2_{n_\cB}} &\leq \frac{1}{\Delta^2}\min\{(\alpha+1)^2, \frac{16\alpha^2}{(\alpha-1)^2} \} + C \: ,\\
\frac{1}{\varepsilon^2_{n_\cA}} &\leq \frac{(\alpha+1)^2}{(\alpha-1)^2}\left( \frac{1}{\Delta^2}\min\{(\alpha+1)^2, \frac{16\alpha^2}{(\alpha-1)^2} \} + C \right) + C \: .
\end{align*}

\end{proof}

\begin{proof}[Proof of Theorem \ref{thm:UCB_alpha_delta_exact}]
We apply Lemma~\ref{lem:generic_nb_pulls_bound} with $\varepsilon_{n} = \sqrt{\frac{2\sigma^2}{n}\log(\frac{2\log^2 n}{\delta})}$, for which concentration holds with probability $1-\tilde{\delta}$ and for which we can take $n_0=3$ and $C = \frac{1}{2\sigma^2\log(2/\delta)} < \frac{1}{2\sigma^2}$. We obtain the following inequalities for all stages prior to the decision,
\begin{align*}
\frac{n_\cB}{\log(\frac{2\log^2 n_\cB}{\delta})} &\leq \frac{2\sigma^2}{\Delta^2}\min\{(\alpha+1)^2, \frac{16\alpha^2}{(\alpha-1)^2} \} + 1 \: ,\\
\frac{n_\cA}{\log(\frac{2\log^2 n_\cA}{\delta})} &\leq \frac{(\alpha+1)^2}{(\alpha-1)^2}\left( \frac{2\sigma^2}{\Delta^2}\min\{(\alpha+1)^2, \frac{16\alpha^2}{(\alpha-1)^2} \} + 1 \right) + 1 \: .
\end{align*}
We introduce the notations $C_1 = \frac{2\sigma^2}{\Delta^2}\min\{(\alpha+1)^2, \frac{16\alpha^2}{(\alpha-1)^2} \} + 1 $ and $C_2 = \frac{(\alpha+1)^2}{(\alpha-1)^2}C_1 + 1$.
At the time of decision,
$$n_\cB \leq C_1 \left(\log\frac{2}{\delta} + 2\log\log(2C_1 \max\{\log\frac{2}{\delta}, \: 2\log(2C_1)\})\right)$$
$$n_\cA \leq C_2 \left(\log\frac{2}{\delta} + 2\log\log(2C_2 \max\{\log\frac{2}{\delta}, \: 2\log(2C_2)\})\right)$$
$$t = n_\cA + n_\cB \leq (C_1 + C_2) \log\frac{1}{\delta} + L \: ,$$
where $L$ regroups the doubly logarithmic terms.

The regret is $\Delta n_B$. Thus with probability greater than $1-\tilde{\delta}$,
\begin{align*}
R_{t_d} \leq \left( \frac{2\sigma^2}{\Delta}\min\{(\alpha+1)^2, \frac{16\alpha^2}{(\alpha-1)^2} \} {+} \Delta \right)  \left(\log\frac{2}{\delta} {+} 2\log\log(2C_1 \max\{\log\frac{2}{\delta}, \: 2\log(2C_1)\})\right)
\end{align*}

\end{proof}

\begin{proof}[Proof of Theorem \ref{thm:ETC_for_two_arms_exact}]
Let $n = t/2$. We consider only even stages $t$. Let $\varepsilon'_n = \sqrt{\frac{4\sigma^2}{n}\log(\frac{\log^2 n}{\delta})}$. As long as no decision is taken we have
$\hat{r}^\cA(n) - \hat{r}^\cB(n) \leq \varepsilon'_n $ .

With probability $1-\tilde{\delta}$, for all $n\geq 1$, we have the concentration inequality $\hat{r}^\cA(n) - \hat{r}^\cB(n) \geq \Delta - \varepsilon'_n $ .

Combining the two inequalities we obtain that as long as no decision is taken, $\varepsilon'_n \leq \frac{\Delta}{2}$. That is,
\begin{align*}
\frac{n}{\log(\frac{\log^2 n}{\delta})} \leq \frac{16\sigma^2}{\Delta^2} \quad
\Rightarrow  \quad n \leq \frac{16\sigma^2}{\Delta^2}(\log\frac{1}{\delta} + 2\log\log(\frac{32\sigma^2}{\Delta^2} \max\{\log\frac{1}{\delta}, \: 2\log(\frac{32\sigma^2}{\Delta^2})\}))\: .
\end{align*}
This bound on $n$ gives both a bound on the regret ($\Delta n$) and on the decision time ($2n$).
\end{proof}

\begin{lemma}
Let $a,b>e$ and $a\geq b$. If $t\geq a + b\log\log(\max\{2a, 2b\log(2b)\})$, then $t\geq a + b\log\log(t)$.
\end{lemma}
\begin{proof}
For $t\geq b$, the function $t\to t - b\log\log(t)$ is increasing. Let $x$ be the solution of $x - b\log\log x = a$. We will show that the proposed $t$ is bigger than $x$.

Case 1: $a \geq b\log\log(x)$. Then $x = a + b\log\log(x)\leq 2a$ and $2a \geq a + b\log\log(2a)$, such that
\begin{align*}
a + b\log\log(2a) \geq a + b\log\log(a + b\log\log(2a)) \: ,
\end{align*}
from which we conclude that $x\leq a + b\log\log(2a)$. Since $t$ is bigger than the latter, it is bigger than $x$.

Case 2: $a \leq b\log\log(x)$. Then $x\leq 2b\log\log x$. If $x>2b\log2b$ then $\frac{x}{\log\log x} > \frac{2b\log2b}{\log\log(2b\log2b)} \geq 2b$. We obtain that $x \leq 2b\log 2b$. This implies that
\begin{align*}
a + b\log\log(2b\log(2b)) \geq a + b\log\log(a+b\log\log(2b\log(2b))) \: ,
\end{align*}
hence $x\leq a+b\log\log(2b\log(2b))$.

In both cases, the proposed $t$ is bigger than $x$, hence it verifies the wanted inequality.
\end{proof}

\section{UCB-MM$_{\alpha}$ and ETC-MM: proofs.}\label{appendix:UCBMM_vs_ETCMM}

\begin{lemma}
Assume that $\sigma_r \geq \frac{\Delta}{\sqrt{2}}$ and $\gamma >4$, and let $n$ be defined by the following equation
\begin{equation}\label{EQ:UCBETCregEq}
\frac{2\Big(\sigma_r^2 + \frac{\sigma_\varepsilon^2\log(en)}{n}\Big)}{n}\log\Big(\frac{36n^4}{\delta}\max\{1,\frac{n\sigma_r^2}{\sigma_\varepsilon^2}\}\Big)  = \frac{\Delta^2}{\gamma}
\end{equation}
then
$$
n \leq \frac{2\gamma\sigma_r^2}{\Delta^2}\log(\frac{1}{\delta})(1+{\scriptscriptstyle \mathcal{O}_\delta}(1)) + \frac{\sigma_\varepsilon^2}{\sigma_r^2}\log\log(\frac{1}{\delta})(1+{\scriptscriptstyle \mathcal{O}_\delta}(1))
$$
\end{lemma}
\begin{proof}
Consider first the following equation
$$
\frac{2 \Sigma^2}{n}\log(\frac{n^5}{\delta}) = C
$$ and  denote by $n_0$ its solution. It follows from straightforward computations that,
$$
n_0 \leq \frac{2 \Sigma^2}{C}\left(\log(\frac{1}{\delta}) + 5\log(\frac{2\Sigma^2}{C})+10 \overline{\log}\Big(\log(\frac{1}{\delta})+5\log\frac{2\Sigma^2}{C}\Big)\right)= \frac{2 \Sigma^2}{C}\log(\frac{1}{\delta})\left(1+{\scriptscriptstyle \mathcal{O}_\delta}(1)\right),
$$
where $\overline{\log}(X)=\max\{5,\log(X)\}$.
We now go back to Equation \eqref{EQ:UCBETCregEq}, and assume for the moment that the solution $n^*$ is such that $n^*\frac{\sigma_r^2}{\sigma^2_\varepsilon} \geq 1$. Moreover, it is clear that $$n^* \geq \frac{2\gamma\sigma_r^2}{\Delta^2}\log\left(\frac{36}{\widetilde{\delta}}\frac{\sigma_r^2}{\sigma_\varepsilon^2} \left( \frac{2\gamma\sigma_r^2}{\Delta^2}\right)^5 \right)=:\underline{n}
$$
As a consequence, if we denote by $\underline{\Sigma}^2=\sigma_r^2 + \frac{\sigma^2_\varepsilon\log(\underline{n})}{\underline{n}}$, then $n^*$ is such that
$$
\frac{2 \underline{\Sigma}^2}{n^*}\log\left(\frac{(n^*)^5}{\underline{\delta}}\right) \leq \frac{\Delta^2}{\gamma}, \quad \text{ where } \underline{\delta} = \frac{36\sigma_r^2}{\widetilde{\delta}\sigma_\varepsilon^2}\, .
$$
So at the end, we have proved that 
$$
n^* = \frac{2\gamma\underline{\Sigma}^2}{\Delta^2}\log(\frac{1}{\delta})(1+{\scriptscriptstyle \mathcal{O}_\delta}(1))= \frac{2\gamma\sigma_r^2}{\Delta^2}\log(\frac{1}{\delta})(1+{\scriptscriptstyle \mathcal{O}_\delta}(1)) + \frac{\sigma_\varepsilon^2}{\sigma_r^2}\log\log(\frac{1}{\delta})(1+{\scriptscriptstyle \mathcal{O}_\delta}(1))
$$
which gives the result. \end{proof}

\begin{corollary}
The decision time of UCB and ETC corresponds respectively to the solution of Equation \eqref{EQ:UCBETCregEq} with $\gamma = 4$ for UCB and $\gamma = 8$ for ETC (and $\gamma=16$ for ETC').
\end{corollary}

\begin{theorem}
Given $\delta>0$ and $\alpha \geq 1$, the decision time of UCB-MM$_{\alpha}$ is such that, with probability at least $1-\widetilde{\delta}$,
$$
t_d \leq    \frac{2\sigma_r^2}{\Delta^2}\Big(2\frac{\alpha^2+1}{(\alpha-1)^2}\big(\min\big\{(\alpha+1)^2,\frac{16\alpha^2}{(\alpha-1)^2}\big\}+\Delta^2\big)+\Delta^2\Big)\log(\frac{1}{\delta})(1+{\scriptscriptstyle \mathcal{O}_\delta}(1))+2\frac{\sigma_r^2}{\sigma_\varepsilon^2}\log\log(\frac{1}{\delta})(1+{\scriptscriptstyle \mathcal{O}_\delta}(1))
$$
Moreover, on the same event, the regret of UCB-MM$_{\alpha}$ at decision time satisfies
$$
R_{t_d} \leq \Big(\frac{8\sigma_r}{\Delta^2}\min\big\{\frac{(\alpha+1)^2}{4}, \frac{4\alpha^2}{(\alpha-1)^2} \big\} + \Delta\Big)\log(\frac{1}{\delta})(1+{\scriptscriptstyle \mathcal{O}_\delta}(1))+\frac{\sigma_r^2}{\sigma_\varepsilon^2}\Delta\log\log(\frac{1}{\delta})(1+{\scriptscriptstyle \mathcal{O}_\delta}(1))
$$
\end{theorem}
\begin{proof}The proof is almost identical to the iid case.  The main difference is the change in error terms. Indeed,  for $n\in\N^*$, we define $$\varepsilon_{n} =\sqrt{\frac{2\Big(\sigma_r^2 + \frac{\sigma_\varepsilon^2\log(en)}{n}\Big)}{n}\log\Big(\frac{36n^4}{\widetilde{\delta}}\max\{1,\frac{n\sigma_r^2}{\sigma_\varepsilon^2}\}\Big)} \: .$$ We apply Lemma~\ref{lem:generic_nb_pulls_bound} with this $\varepsilon_n$, $n_0=0$ and $C=1$. It yields that for all stages prior to the decision,
\begin{align*}
\frac{1}{\varepsilon^2_{n_\cB}}  &\leq \frac{1}{\Delta^2}\min\{(\alpha+1)^2, \frac{16\alpha^2}{(\alpha-1)^2} \} + 1 \: ,\\
\frac{1}{\varepsilon^2_{n_\cA}} &\leq \frac{(\alpha+1)^2}{(\alpha-1)^2}\left( \frac{1}{\Delta^2}\min\{(\alpha+1)^2, \frac{16\alpha^2}{(\alpha-1)^2} \} + 1 \right) + 1 \: .
\end{align*}
We now introduce the notations $\gamma_1 = \min\{(\alpha+1)^2, \frac{16\alpha^2}{(\alpha-1)^2} \} + \Delta^2 $ and $\gamma_2 = \frac{(\alpha+1)^2}{(\alpha-1)^2}\gamma_1 + \Delta^2$.
At the time of decision,
\begin{align*}
n_\cB &= \frac{2\gamma_1\sigma_r^2}{\Delta^2}\log(\frac{1}{\delta})(1+{\scriptscriptstyle \mathcal{O}_\delta}(1))+\frac{\sigma_r^2}{\sigma_\varepsilon^2}\log\log(\frac{1}{\delta})(1+{\scriptscriptstyle \mathcal{O}_\delta}(1))\\
n_\cA &= \frac{2\gamma_2\sigma_r^2}{\Delta^2}\log(\frac{1}{\delta})(1+{\scriptscriptstyle \mathcal{O}_\delta}(1))+\frac{\sigma_r^2}{\sigma_\varepsilon^2}\log\log(\frac{1}{\delta})(1+{\scriptscriptstyle \mathcal{O}_\delta}(1))\\
t_d=n_\cA+n_\cB &=  \frac{2(\gamma_2+\gamma_1)\sigma_r^2}{\Delta^2}\log(\frac{1}{\delta})(1+{\scriptscriptstyle \mathcal{O}_\delta}(1))+2\frac{\sigma_r^2}{\sigma_\varepsilon^2}\log\log(\frac{1}{\delta})(1+{\scriptscriptstyle \mathcal{O}_\delta}(1))\\
& = \frac{2\sigma_r^2}{\Delta^2}\Big(2\frac{\alpha^2+1}{(\alpha-1)^2}\big(\min\big\{(\alpha+1)^2,\frac{16\alpha^2}{(\alpha-1)^2}\big\}+\Delta^2\big)+\Delta^2\Big)\log(\frac{1}{\delta})(1+{\scriptscriptstyle \mathcal{O}_\delta}(1))\\&\hspace{3cm}+2\frac{\sigma_r^2}{\sigma_\varepsilon^2}\log\log(\frac{1}{\delta})(1+{\scriptscriptstyle \mathcal{O}_\delta}(1))
\end{align*}
As a consequence, on the same event,
\begin{align*}
R_{t_d} \leq \Big(\frac{8\sigma_r}{\Delta^2}\min\big\{\frac{(\alpha+1)^2}{4}, \frac{4\alpha^2}{(\alpha-1)^2} \big\} + \Delta\Big)\log(\frac{1}{\delta})(1+{\scriptscriptstyle \mathcal{O}_\delta}(1))+\frac{\sigma_r^2}{\sigma_\varepsilon^2}\Delta\log\log(\frac{1}{\delta})(1+{\scriptscriptstyle \mathcal{O}_\delta}(1))
\end{align*}

\end{proof}

\section{Static population}

In the static setting, all users are allocated to populations $\cA$ and $\cB$ from the beginning of the test, and we only consider the case where the size of both populations are equal, even though the generalization to  different population size is almost straightforwaed.

\paragraph{Finite fixed horizon:}
The first baseline is to wait until some horizon $T$ and perform a  statistical test based on a confidence bound on the uplift $\Delta$.

\begin{proposition}
In the static setting, the following holds with probability at least $1-\delta$,
\begin{align}
  \Delta - (\hat{r}^\cA-\hat{r}^\cB_T) \leq
\sqrt{\frac{8\left(\sigma_r^2 + \frac{\sigma_\varepsilon^2}{T}\right)\log \frac{1}{\delta}}{n}}
 .
\nonumber
\end{align}
Therefore, the procedure waiting until the horizon $T$ to select the $\cB$ if $\hat{\Delta}_T$ is greater than the r.h.s. term has a linear regret of $R(T) = n\Delta/2$ and is guaranteed to be $(\delta,T)$-PAC 
$$
n \geq \frac{32(\sigma^2_r+\frac{\sigma^2_\varepsilon}{T})}{\Delta^2}\log \frac{1}{\delta} \: .
$$
\label{thm:finite_horizon}
\end{proposition}
The proof is a direct consequence of standard concentration inequalities.

\paragraph{Adaptive  decision time}
Instead of waiting for a fixed arbitrary horizon $T$, the decision can often be taken  before, at the cost of using maximal concentration inequalities, that are valid at all stages.

\begin{theorem}
\label{theorem_ETC}
In the static setting, it holds that, for all $t \in \N$ and with probability at least $1-\widetilde{\delta}$,
\begin{align*}
 \Delta - (\hat{r}^\cA_t- \hat{r}^\cB_t)
\leq \sqrt{\frac{8\sigma_r^2 \log\frac{2}{\delta}}{n}}
+
\sqrt{\frac{8\sigma_\varepsilon^2 \log\frac{3\log^2(t)}{\delta}}{tn}}\ .\end{align*}
\label{thm:max_bernstein_static}
\end{theorem}
As a consequence, ETC can take a correct decision with probability at least $1-\delta$ if the number of users $n$ is greater than $\frac{32\sigma_r^2}{\Delta^2}\log \frac{2}{\delta}$ and then, if we denote by $\eta:= \Delta-\sqrt{\frac{32\sigma^2_r}{n}\log(\frac{1}{\delta})}$, the decision will be taken before the time step
$$
\frac{32\sigma_\varepsilon^2}{\eta^2}\log\big(\frac{1}{\delta}\big)(1+{\scriptscriptstyle \mathcal{O}_\delta}(1)).
$$

\end{document}